\documentclass[]{fairmeta}
\usepackage{amssymb}            %
\usepackage{mathtools}          %
\usepackage{mathrsfs}           %
\usepackage{graphicx}           %
\usepackage{subcaption}         %
\usepackage[space]{grffile}     %
\usepackage{url}                %
\usepackage{lipsum}             %
\usepackage[ruled,vlined]{algorithm2e}
\usepackage{amsthm}
\usepackage{thmtools, thm-restate}
\usepackage{wrapfig}
\usepackage{tabularx}
\usepackage{multirow}
\usepackage{graphicx}
\makeatletter
\def\thm@space@setup{%
  \thm@preskip=2pt
  \thm@postskip=2pt %
}
\makeatother

\declaretheorem{proposition}

\theoremstyle{plain}

\theoremstyle{definition}

\theoremstyle{remark}

\usepackage{amsmath,amsfonts,bm}

\def\eqref#1{equation~\ref{#1}}

\def\1{\bm{1}}

\DeclareMathAlphabet{\mathsfit}{\encodingdefault}{\sfdefault}{m}{sl}
\SetMathAlphabet{\mathsfit}{bold}{\encodingdefault}{\sfdefault}{bx}{n}

\def\gA{{\mathcal{A}}}

\def\gS{{\mathcal{S}}}

\def\gZ{{\mathcal{Z}}}

\newcommand{\E}{\mathbb{E}}

\newcommand{\R}{\mathbb{R}}

\newcommand{\LOLA}{\texttt{LoLA}}
\newcommand{\RELA}{\texttt{ReLA}}
\usepackage{amsmath}
\usepackage{hyperref}
\usepackage{url}
\usepackage{booktabs}
\usepackage{graphicx}
\usepackage{subcaption}
\usepackage{bbm}
\usepackage{soul}
\usepackage{makecell}

\title{Fast Adaptation with Behavioral Foundation Models}

\author[1,*]{Harshit Sikchi}
\author[2]{Andrea Tirinzoni}
\author[2]{Ahmed Touati}
\author[2]{Yingchen Xu}
\author[2]{Anssi Kanervisto}
\author[3]{Scott Niekum}
\author[1]{Amy Zhang}
\author[2]{Alessandro Lazaric}
\author[2]{Matteo Pirotta}

\affiliation[1]{The University of Texas at Austin}
\affiliation[2]{FAIR at Meta}
\affiliation[3]{UMass Amherst}

\contribution[*]{Work done during an internship at FAIR, Meta}

\abstract{
Unsupervised zero-shot reinforcement learning (RL) has emerged as a powerful paradigm for pretraining behavioral foundation models (BFMs), enabling agents to solve a wide range of downstream tasks specified via reward functions in a zero-shot fashion, i.e., without additional test-time learning or planning. This is achieved by learning self-supervised task embeddings alongside corresponding near-optimal behaviors and incorporating an inference procedure to directly retrieve the latent task embedding and associated policy for any given reward function. Despite promising results, zero-shot policies are often suboptimal due to errors induced by the unsupervised training process, the embedding, and the inference procedure. In this paper, we focus on devising \emph{fast adaptation} strategies to improve the zero-shot performance of BFMs in few steps of online interaction with the environment, while avoiding any performance drop during the adaptation process. Notably, we demonstrate that existing BFMs learn a set of skills containing more performant policies than those identified by their inference procedure, making them well-suited for fast adaptation. Motivated by this observation, we propose both actor-critic and actor-only fast adaptation strategies that search in the low-dimensional task-embedding space of the pre-trained BFM to rapidly improve the performance of its zero-shot policies on any downstream task. Notably, our approach mitigates the initial “unlearning” phase commonly observed when fine-tuning pre-trained RL models. We evaluate our fast adaptation strategies on top of four state-of-the-art zero-shot RL methods in multiple navigation and locomotion domains. Our results show that they achieve 10-40\% improvement over their zero-shot performance in a few tens of episodes, outperforming existing baselines.
}

\date{\today}
\correspondence{Harshit Sikchi at \email{hsikchi@utexas.edu}}

\begin{document}

\maketitle
\section{Introduction}

Unsupervised (or self-supervised) pre-training has emerged as one of the key ingredients behind the recent breakthroughs in computer vision and language modeling~\citep[e.g.,][]{radford2019language,devlin2019bert,touvron2023llama,caron2021emerging}. This technique allows utilizing large datasets of unlabeled data samples to learn generalizable representations that can be later fine-tuned for various downstream applications~\citep{zhai2023sigmoid,brown2020language,driess2023palm}. For instance, language models are pre-trained on internet-scale data with a next-token prediction objective and later fine-tuned for desired applications using high-quality examples. How to transpose this approach to reinforcement learning (RL) to train agents that can efficiently solve sequential decision-making problems is an open research question of paramount importance. Going beyond the tabula-rasa paradigm of classic RL requires an unsupervised pre-training objective and the ability to efficiently fine-tune or adapt pre-trained representations for downstream tasks. Recent developments in unsupervised RL propose various objectives to learn a repertoire of skills on top of reward-free data from the environment \citep{gregor2016variational,wu2018laplacian,hansen2019fast,liu2021aps,eysenbach2018diversity,zahavy2022discovering,park2023metra}. Some of these methods are named ``zero-shot'', in the sense that they additionally provide a procedure to infer a performant policy for any given task specified by reward functions~\citep{zeroshot,park2024foundation,agarwal2024proto, cetin2024finer}, demonstrations~\citep{pirotta2024fast, tirinzoni2025zeroshot}, or videos/language~\citep{sikchi2024rl}. The resulting pre-trained agents are commonly referred to as Behavioral Foundation Models \citep[BFMs,][]{pirotta2024fast,tirinzoni2025zeroshot}.

\begin{figure}[t]
    \centering
    \begin{minipage}{0.60\textwidth}
        \centering
        \includegraphics[width=\textwidth]{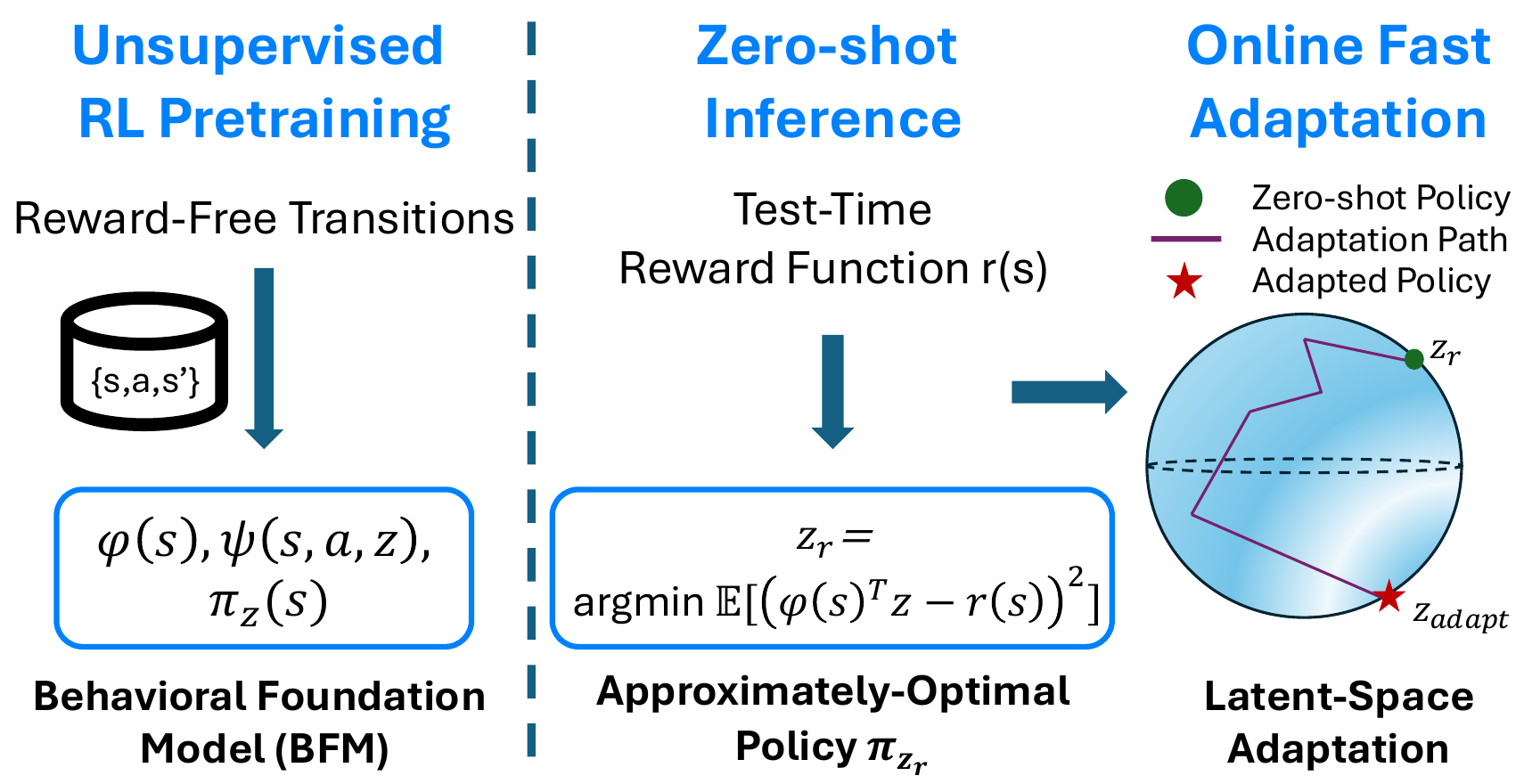}
    \end{minipage}%
    \hspace{5mm}
    \begin{minipage}{0.35\textwidth}
        \centering
        \includegraphics[width=\textwidth]{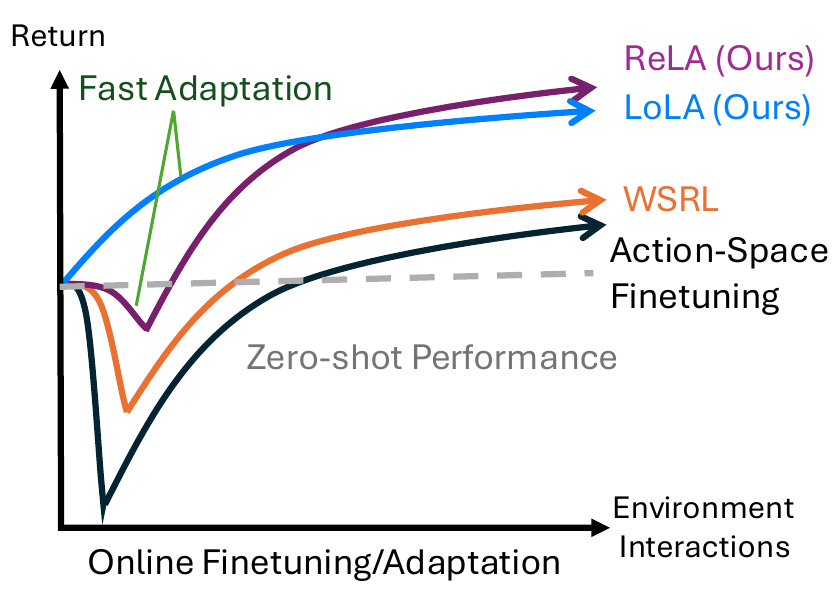}
    \end{minipage}
    \caption{\textbf{Overview of our method}: Unsupervised zero-shot RL methods provide us with an initial policy $\pi_{z_r}$; we propose a way to leverage the latent space of learned policies as well as the pre-trained critic to rapidly adapt and improve $\pi_{z_r}$ on few task-specific environment interactions. Right: Illustrative summary of our results.}
    \label{fig:sidebyside}
    \vspace{-0.5cm}
\end{figure}

Zero-shot methods commonly pre-train two components: (1) a \emph{state representation} $\varphi : \mathcal{S} \rightarrow \mathbb{R}^d$ that embeds state observations $s \in \mathcal{S}$ into a d-dimensional vector $\varphi(s)$, and (2) a space $\{\pi_z\}$ of \emph{policies} parameterized by a latent vector $z\in\mathbb{R}^d$. The representation $\varphi$ defines the set of all linear reward functions in $\varphi$, i.e., $\tilde{r}_z(s) = \varphi(s)^T z$ for all $z\in\mathbb{R}^d$, which in turn is used as a \emph{self-supervised} objective function for the policy space: for each $z\in\mathbb{R}^d$, the policy $\pi_z$ is trained to be approximately optimal for the reward $\tilde{r}_z$. Given a reward function $r(s)$ at test time, a zero-shot policy $\pi_{z_r}$ can be obtained by projecting $r$ onto the pre-trained state features $\varphi$ through linear regression on top of the training data, hence approximating $r(s) \simeq \varphi(s)^T z_r$.

Although this inference method has proven effective in producing reasonable policies, it suffers from two main limitations yielding sub-optimal performance. First, the embedding $\varphi$ is learned using unsupervised losses encoding inductive biases\footnote{For instance, some methods rely on low-rank assumptions in the policy dynamics~\citep{fballreward,agarwal2024proto}, while others focus only on goal-reaching behaviors~\citep{park2024foundation}} that may not be suitable for the downstream tasks of interest. As a result, the projection of the reward function onto $\varphi$ may remove crucial aspects of the task specification thus preventing from finding the optimal policy for the original reward. In an extreme scenario, if a reward function lies in the orthogonal subspace of the features' linear span, its projection onto these features becomes zero, making it uninformative.
Second, BFMs are typically trained on task-agnostic datasets that may have poor coverage of the rewarding states relevant to the specific task. This limitation can result in zero-shot inference failing to accurately represent these states and ultimately hinder the learning of a good policy.

While the suboptimality of unsupervised pre-training of large models is somewhat unavoidable, it is natural to wonder whether these limitations can be overcome once a downstream reward function is given and the agent has online access to the environment. In this paper we focus on devising \textit{fast adaptation} strategies that improve zero-shot performance of BFMs \textbf{1)} \textit{rapidly}, i.e., in a handful of online episodes, and \textbf{2)} \textit{monotonically}, i.e., avoiding any performance drop during the adaptation process. This motivates the main question of this work:
\begin{center}
  \textit{Does the policy space of a pre-trained BFM contain better behaviors than those returned by zero-shot inference? If so, can we retrieve them with few task-specific environment interactions?} 
\end{center}
To address this question, we propose searching over the latent space $\gZ$ using a limited number of online task-specific interactions with the environment (cf. Figure~\ref{fig:sidebyside}). We introduce two algorithms that leverage the latent space and pre-trained components from BFMs to enable \emph{fast adaptation} of their zero-shot policies: (1) Residual Latent Adaptation (ReLA), an off-policy actor-critic approach that trains a small \emph{residual critic} to compensate for the reward projection errors, and Lookahead Latent Adaptation (LoLA), a hybrid actor-only approach that combines on-policy optimization while bootstrapping the frozen critic from the pre-trained BFMs. %

We perform an extensive empirical evaluation on 5 domains with a total of 64 tasks spanning low-dimensional and high-dimensional problems with increasing complexity, including a whole-body humanoid control problem with a wide range of 45 diverse reward-based behaviors. We demonstrate the effectiveness of our proposed algorithms on four state-of-the-art BFMs: FB \citep{fballreward}, HILP \citep{park2024foundation}, PSM \citep{agarwal2024proto} and FB-CPR~\citep{tirinzoni2025zeroshot}. In particular, we answer the above question affirmatively: our fast adaptation algorithms achieve 10-40\% improvement over the BFMs zero-shot performance in only a few episodes (Figure~\ref{fig:main_summary_experiment} and~\ref{fig:qualitative-adaptation}), while outperforming existing baselines. Moreover, we show that \LOLA{} avoids any initial drop of performance, a phenomenon commonly observed by numerous prior works on fine-tuning RL policies~\citep{nair2020awac,nakamoto2023cal,luo2023finetuning,zhou2024efficient}.

\section{Preliminaries}
\begin{figure}
    \centering
    \includegraphics[width=1.0\linewidth]{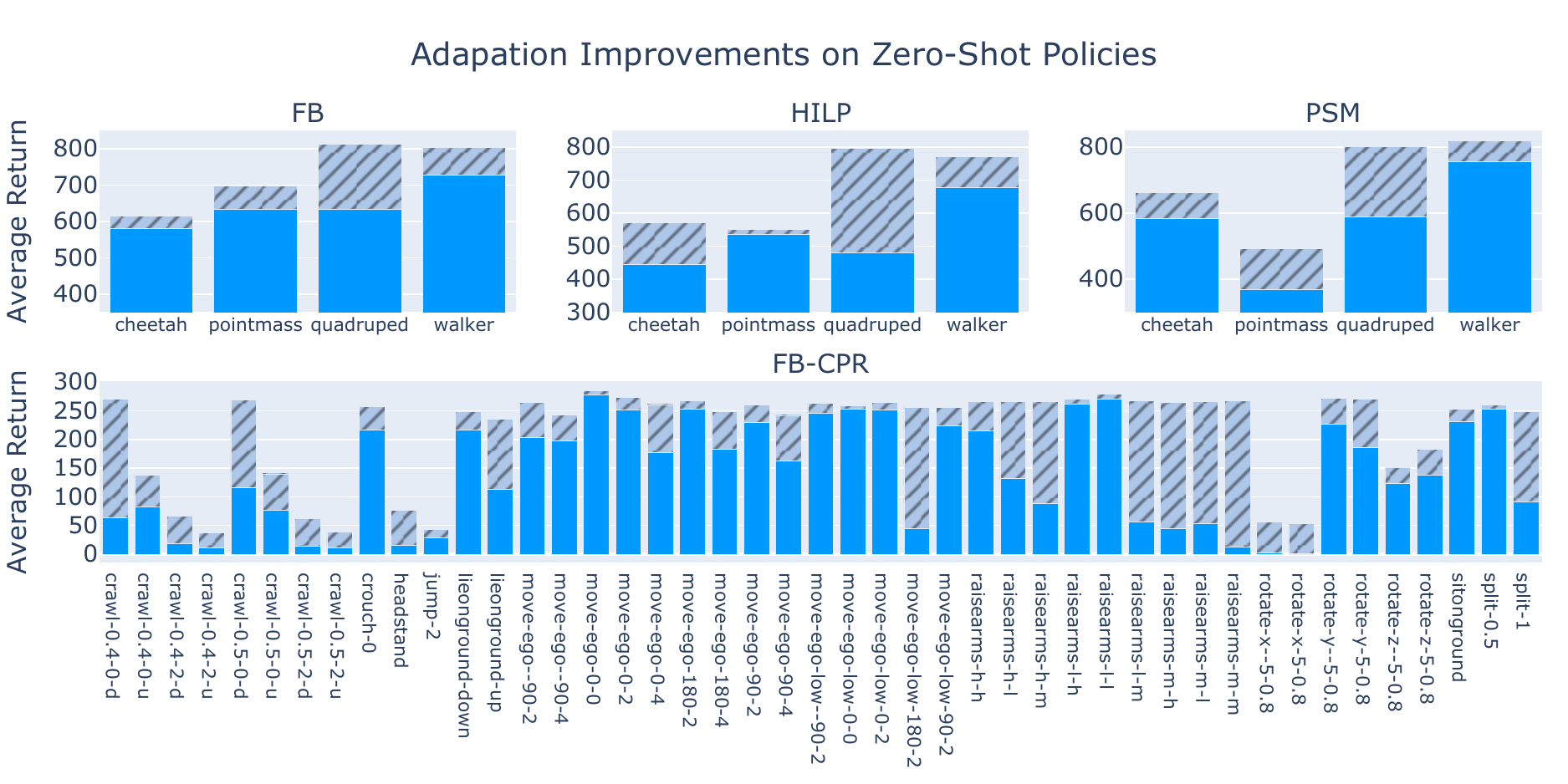}
    \caption{Performance comparison of zero-shot policy vs adapted policy in the BFM's latent space after 200 episodes. The shaded region shows the improvement of the adapted policies averaged across tasks.}
    \label{fig:main_summary_experiment}
    \vspace{-0.3cm}
\end{figure}

\textbf{Markov decision process.} We consider a reward-free Markov decision process (MDP)~\citep{puterman2014markov,sutton2018reinforcement} which is defined as a tuple $\mathcal{M}=(\gS,\gA,P,d_0, \gamma)$, where $\gS$ and $\gA$ respectively denote the state and action spaces, $P$ denotes the transition kernel with $P(s'|s,a)$ indicating the probability of transitioning from $s$ to $s'$ by taking action $a$, $d_0$ denotes the initial state distribution and $\gamma \in (0,1)$ specifies the discount factor. A policy $\pi$ is a function $\pi: \gS \rightarrow \Delta(\gA)$ mapping a state s to probabilities of action in $\gA$. We denote by $\text{Pr}(\cdot \mid s, a, \pi)$ and $\mathbb{E}[\cdot \mid s, a, \pi]$ the
probability and expectation operators under state-action sequences $(s_t, a_t)_{t \geq 0}$ starting at $(s, a)$ and
following policy $\pi$ with $s_t \sim P( \cdot \mid s_{t-1}, a_{t-1})$ and $a_t \sim \pi(\cdot \mid s_t)$. Given any reward function $r: \gS \rightarrow \R$, the Q-function of $\pi$ for $r$ is $Q^\pi_r(s, a) := \sum_{t \geq 0} \gamma^{t}\mathbb{E}[r(s_{t+1}) \mid s, a, \pi]$.

\textbf{Successor measures and features.} The \textit{successor measure}~\citep{dayan1993improving, blier2021learning} of state-action $(s, a)$ under a policy $\pi$ is the
(discounted) distribution of future states obtained by taking action $a$ in state $s$ and following policy $\pi$
thereafter:
\begin{equation}
    M^\pi(X \mid s, a) := \sum_{t \geq 0} \gamma^t \text{Pr}(s_{t+1} \in X \mid s, a, \pi) \quad \forall X \subset \gS.
\end{equation}
Importantly, successor measures disentangle the dynamics of the MDP and the reward function: for
any reward $r$ and policy $\pi$, the Q-function can be expressed linearly as $Q^\pi_r = M^\pi r$. 

Given a feature map $\varphi: \gS \rightarrow \R^d$ that embeds states into a $d$-dimensional space, the \textit{successor features}~\citep{barreto2017successor} is the expected discounted sum of features:
\begin{equation}
    \psi^\pi(s, a) := \sum_{t\geq 0} \gamma^t \mathbb{E}[\varphi(s_{t+1}) \mid s, a, \pi].
\end{equation}
Successor features and measures are related: by definition, $ \psi^\pi(s, a) = \int_{s} M^\pi(\mathrm{d}s' \mid s, a) \varphi(s')$. For any reward function in the linear span of $\varphi$, \textit{i.e.}, $r(s) =  \omega^\top \varphi(s)$ where $\omega$ is a weight vector in $\R^d$, the Q-function can be expressed compactly as $Q^\pi_r(s, a) = \omega^\top \psi^\pi(s, a)$.

\begin{figure}
    \centering
    \includegraphics[width=1.0\linewidth]{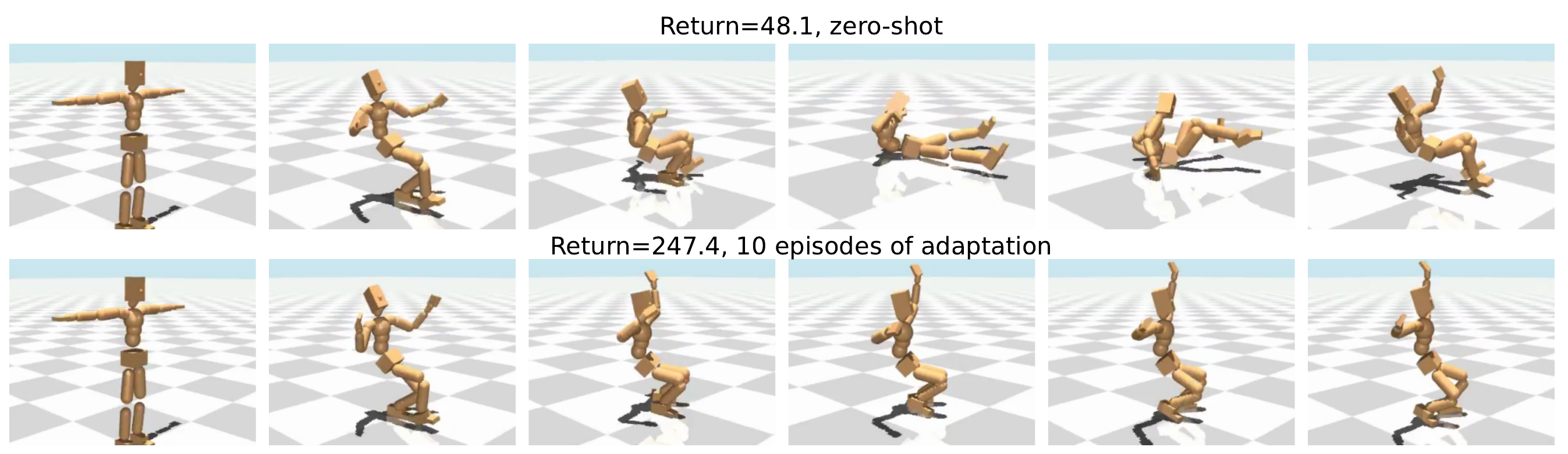}
    \caption{Qualitative difference in behaviors in 10 episodes of adaptation in HumEnv environment for the task move-ego-low-180-2 with our method LoLA.}
    \label{fig:qualitative-adaptation}
\end{figure}

\textbf{Behavioral foundation models.} A \textit{behavioral foundation
model}, for a given MDP, is an agent that can be trained in an unsupervised fashion using reward-free transitions and yet can produce approximately
optimal policies for a large class of reward functions $r$ specified at test time, without performing additional learning or planning. In this work, we focus on zero-shot RL agents that are based on successor features and forward-backward representations.  

\textit{Universal successor features} (USFs)~\citep{borsa2018universal} provide a generic framework for zero-shot RL. Given a feature map $\varphi$, USFs learn the successor features of a particular family of policies $\pi_z$ parameterized by latent variables $z \in \gZ \subset \R^d$:
\begin{equation}
    \psi(s, a, z) = \mathbb{E}[\sum_{t\geq 0} \gamma^t \varphi(s_{t+1}) \mid s, a, \pi_z ], \quad \pi_z(s) = \arg\max_{a} \psi(s, a, z)^\top z.
\end{equation}
At test time, once a reward function $r$ is specified, a reward-maximizing policy is inferred by performing a linear regression of $r$ onto the features $\varphi$. In particular, we estimate $z_r= \arg\min_{z}\mathbb{E}_{s \sim \rho}[(r(s) - \varphi(s)^\top z)^2] = \mathbb{E}_{s \sim \rho}[\varphi(s) \varphi(s)^\top]^{-1} \mathbb{E}_{s \sim \rho}[\varphi(s) r(s)]$ where $\rho$ is some dataset distribution over states. Then we return the pre-trained policy $\pi_{z_r}$. This policy is guaranteed to be optimal if the reward is in the linear span of the features $\varphi$~\citep{borsa2018universal}. Although USF is a generic framework, it requires specifying a training criterion to learn the basic features $\varphi$. \citet{zeroshot} compare several choices of unsupervised representation learning objectives across various empirical problems. In this work, we focus on two recent state-of-the-art feature learning methods for zero-shot RL: \textit{Hilbert representations} (HILP)~\citep{park2024foundation} and \textit{proto successor measures} (PSM)~\citep{agarwal2024proto}. HILP constructs features $\varphi$ such that the distance $\| \varphi(s) - \varphi(s')\|$ between a state pair $(s, s')$ encodes the optimal value function of reaching the state $s'$ starting at $s$. PSM proposes to build the features $\varphi$ by learning an \textit{affine decomposition} of the successor measure for a discrete codebook of policies, \textit{i.e.}, $M^{\pi_u}(\mathrm{d}s'\mid s, a) / \rho(\mathrm{d}s') \approx \phi(s, a)^\top\varphi(s') w(u)+ b(s, a, s')$, where $\phi, w$ and $b$ are vector-valued functions and where $\pi_u$ is a deterministic policy that outputs an action in state $s$ as a realization of the uniform distribution, determined by the random seed $u$.

\textit{Forward-backward representations} (FB)~\citep{fballreward} provide an alternative framework for zero-shot RL. Unlike USFs which use two separate criteria to learn features and their successor features, FB avoid the state featurization step and employ a single objective to learn a finite-rank decomposition of the successor measure for various policies. Namely, FB pre-train two representations $F: \gS \times \gA \times \gZ \rightarrow \R^d$ and $B: \gS \rightarrow \R^d$  such that:
\begin{equation}
    F(s, a, z)^\top B(s') \rho(\mathrm{d}s') \approx M^{\pi_z}(\mathrm{d}s'\mid s, a), \quad \pi_z(s) = \arg\max_{a}F(s, a, z)^\top z.
\end{equation}
FB representations are related to USFs, as $F(s, a, z)$ represents the successor features of $\mathbb{E}_{s \sim \rho}[B(s)B(s)^\top]^{-1} B(s)$~\citep{zeroshot}. In the sequel, to standardize the notations with the USFs, we will denote $\psi(s, a, z)=F(s, a, z)$ and $\varphi(s) = \mathbb{E}_{s \sim \rho}[B(s)B(s)^\top]^{-1} B(s)$.

\textit{Forward-Backward representations with Conditional Policy Regularization} (FB-CPR)~\citep{tirinzoni2025zeroshot} is an online variant of FB that grounds the unsupervised policy learning toward imitating observation-only unlabeled behaviors.

\section{Fast Adaptation for Behavioral Foundation Models} \label{sec:methods}

In this section, we introduce our two approaches for fast adaptation of pre-trained BFMs: an off-policy actor-critic algorithm (Section~\ref{sec:rela}), and a hybrid on-policy actor-only algorithm (Section~\ref{sec:lola}).

\subsection{ReLA: Residual Latent Adaptation}
\label{sec:rela}
Given a reward function $r$, ReLA begins with the latent variable $z = z_r$ inferred by the zero-shot procedure and uses an off-policy actor-critic approach to gradually update $z$ towards better performance. The overall algorithm uses a standard online training procedure, interleaving between critic and actor updates (as described below), while gathering reward-labeled transitions in a replay buffer $\mathcal{D}_{\text{online}}$ through online interactions with the environment.

\textbf{Residual critic learning.} Instead of training a critic from scratch to model the Q-function of the policy $\pi_z$ currently being learned for the reward $r$, ReLA uses a residual critic to correct for the reward projection error. This is made possible by the following decomposition:
\begin{align}
    Q^{\pi_z}_r(s, a) & = Q^{\pi_z}_{\varphi^\top z_r}(s, a) + Q^{\pi_z}_{r - \varphi^\top z_r}(s, a) \nonumber\\
    & = \psi(s, a, z)^\top z_r + Q^{\pi_z}_{r - \varphi^\top z_r}(s, a)
    \label{eq:residual_decomposition}
\end{align}
where the last equality holds because $\psi$ is pre-trained to estimate the successor features of $\varphi$ and the projected reward $\varphi^\top z_r$ lies in the span of $\varphi$. Consequently, ReLA considers a network $Q^{\text{residual}}(s, a; \theta)$ parametrized by weights $\theta$  and trains it via off-policy TD learning so that  $\psi(s, a,z_r)^\top z_r + Q^{\text{residual}}(s, a; \theta)$ approximates the Q-function $Q^{\pi_z}_r(s, a)$, while keeping the \textit{base Q-function} $\psi(s, a,z_r)^\top z_r$ frozen. In practice, we shall use much smaller networks for the residual critic than for the pre-trained successor features, with the main intuition being that we only need to compensate for some projection error. For a more in-depth treatment of the Q-function decomposition we refer the readers to Appendix~\ref{ap:residual_motivation}.

\textbf{Latent actor update.} ReLA updates the latent variable $z$ using standard policy-gradient ascent, with the key difference being that the gradient is computed only with respect to $z$, while keeping the pre-trained actor parameters fixed,
\begin{equation}
    \label{eq:latent_policy_update}
    \nabla_z \mathbb{E}_{s \sim \mathcal{D}_{\text{online}}}[ \psi(s, \pi_z(s),z_r)^\top z_r + Q^{\text{residual}}(s, \pi_z(s); \theta) ],
\end{equation}
The main advantage over optimizing the whole actor network is that we only need to search in a low-dimensional space (in practice, $z$ has in the order of hundreds of components, while the actor network of a BFM has in the order of millions of parameters).

\subsection{LoLA: Lookahead Latent Adaptation}
\label{sec:lola}

Although ReLA can take advantage of off-policy data collected in the replay buffer, it requires learning an additional residual network. Therefore, ReLA demands a certain budget of transitions and updates to mitigate the distribution shift issue~\citep{luo2023finetuning} when learning the Q-function, which may impede improvements during the very early stages of adaptation. On the other hand, a purely on-policy approach
will require rolling out entire trajectories under the current policy to estimate Monte Carlo returns $\sum_{t=0}^T \gamma^t r_{t}$ (where $T$ is the episode length), and thus incur many environment interactions in the process. Alternatively, we propose \textbf{Lo}okahead \textbf{L}atent \textbf{A}daptation algorithm (LoLA) that uses fixed-horizon on-policy rollouts with a frozen terminal value function obtained from the BFM. LoLA parameterizes a policy over the latent space as a normal distribution $\pi_{\mu, \sigma}=\mathcal{N}(\mu,\sigma)$ with trainable mean $\mu$ (initialized with $\mu=z_r$), and fixed diagonal covariance $\sigma$. The pre-trained successor features from BFM are used compute the estimate of a terminal value-function, thus estimating the n-step lookahead return of policy $\pi_z$ starting from state $s_0$ as $R^n(s_0, z) = \sum_{t=0}^{n-1} \gamma^t r(s_{t+1}) + \gamma^{n} \psi(s_{n+1}, \pi_z(s_{n+1}), z)^\top z_r$.

Moreover, to further improve learning, LoLA  incorporates the variance reduction strategy of leave-one-out baseline~\citep{kool2019buy}. This baseline has recently been shown to be empirically effective for fine-tuning large language models~\citep{ahmadian2024back}. This leads to the following final gradient estimate\footnote{In practice we work with z normalized on hypersphere using projected gradient descent}:
\begin{equation}
\resizebox{0.97\textwidth}{!}{
$\mathbb{E}_{s_0 \sim \nu}\left[\frac{1}{k} \sum_{i=1}^k\left( R\bigl(s_0,z_i\bigr)
\;-\;
\frac{1}{k-1} \sum_{\substack{j=1 \\ j \neq i}}^k
R\bigl(s_0,z_j\bigr)\right)\,\nabla_{\mu} \log \pi_{\mu, \sigma}\bigl(z_i\bigr) \right]
~~
\text{for}
~~
z_1, \dots, z_k \sim \pi_{\mu, \sigma}(\cdot)\,$
}
\label{eq:lola_loss}
\end{equation}
where $s_0$ is sampled from the distribution $\nu$ defined as mixture between the environment's initial distribution $d_0$ and the online replay buffer distribution $\mathcal{D}_{\text{online}}$. For each sampled starting state $s_0$, we sample $k$ latent variables $\{ z_i\}_{i \in [k]}\sim \pi_{\mu, \sigma}$ and generate $k$ trajectories of length $n$ $(s^{(i)}_0, a^{(i)}_0, s^{(i)}_1,  a^{(i)}_1, \ldots, s^{(i)}_n)$ by following the policy $\pi_{z_i}$. Computing the gradient requires the ability to reset of any state in support of distribution $\nu$, which includes the states encountered during online adaptation.

\section{Experimental Results}
The goal of our experiments is to study how well latent policy adaptation works on top of existing BFMs. %
We perform several ablations to understand the efficacy of the proposed methods and evaluate our design choices. Precisely, 
\emph{1)} Can we find better policies by online latent policy adaptation compared to the zero-shot policies? Or, equivalently, is the latent policy space easy to search over? 
\emph{2)} How important is to leverage BFMs properties (e.g., Q-function estimate, zero-shot policy initialization)? 
\emph{3)} What are the critical limitations of the zero-shot inference process?

\textbf{Experimental setup.}
We investigate these questions by leveraging 4 different BFMs: FB, HILP, PSM and FB-CPR. 
While FB, HILP and PSM are trained offline, FB-CPR learns through online environmental interactions and it is regularized towards expert trajectories.
We consider four environments from the DeepMind Control suite~\citep{tassa2018deepmind} and train the BFMs on an exploratory dataset obtained from ExoRL~\citep{exorl}.\footnote{We consider the dataset collected by running RND~\citep{DBLP:conf/iclr/BurdaESK19}.} Further, we leverage the FB-CPR model released by~\citet{tirinzoni2025zeroshot} for the \texttt{HumEnv} environment, a high-dimensional humanoid agent. Overall, we consider 7 tasks for \texttt{Pointmass}, 4 for \texttt{Cheetah}, 4 for \texttt{Quadruped}, 4 for \texttt{Walker}, and 45 tasks for \texttt{HumEnv}. Detailed information about the pre-training phase can be found in Appendix~\ref{ap:experiments}.

\textbf{Protocol and baselines.} While the paper focuses on adaptation in the latent policy space, we also investigate the common class of approaches for fine-tuning in action space (i.e., updating all policy network parameters) with zero-shot initialization~\citep{nair2020awac,nakamoto2023cal}.  In particular, we consider a TD3-based algorithm~\citep{fujimoto2018addressing} that train a critic from scratch and an actor initialized using the zero-shot policy (\textsc{TD3 (I)}).\footnote{We also tested vanilla RLOO~\citep{kool2019buy} but did not get good results and decided not to report it.} Since collecting a few on-policy trajectories before starting updating the critic and the actor proved to be effective for offline to online adaptation, a strategy called Warm-Start RL (\textsc{WSRL})~\citep{zhou2024efficient}, we additionally consider this component for action-based algorithms. We further ablate several design choices (e.g., zero-shot initialization, bootstrapped critic) in Section~\ref{subsec:ablation} and, in Appendix~\ref{app:additional.exp} we report variations of our algorithms that operate by directly updating the parameters of the policy. See Table~\ref{tab:main.algorithms} for a complete list of algorithm variations.

We use a comparable architecture and hyperparameter search for all algorithms. For each BFM, we report the performance on the set of hyperparameters that performed best across all tasks and domains. We train all the online adaptation algorithms for $300$ episodes and we use $5$ seeds for each experiment. Evaluation is done by averaging results over $50$ episodes.  We also use TD3 as base off-policy algorithm for implementing \RELA{}. When using residual critic we use a small 2-layers MLP with hidden dimension $64$, while when we learn the critic from scratch we use a 2-layers MLP with hidden dimension $1024$. The policy has always the same size as the BFM policy. We provide further implementation details in Appendix~\ref{ap:experiments}.

\subsection{Do ReLA and LoLA enable fast adaptation?}

\begin{figure}
    \centering
    \includegraphics[width=1.0\linewidth]{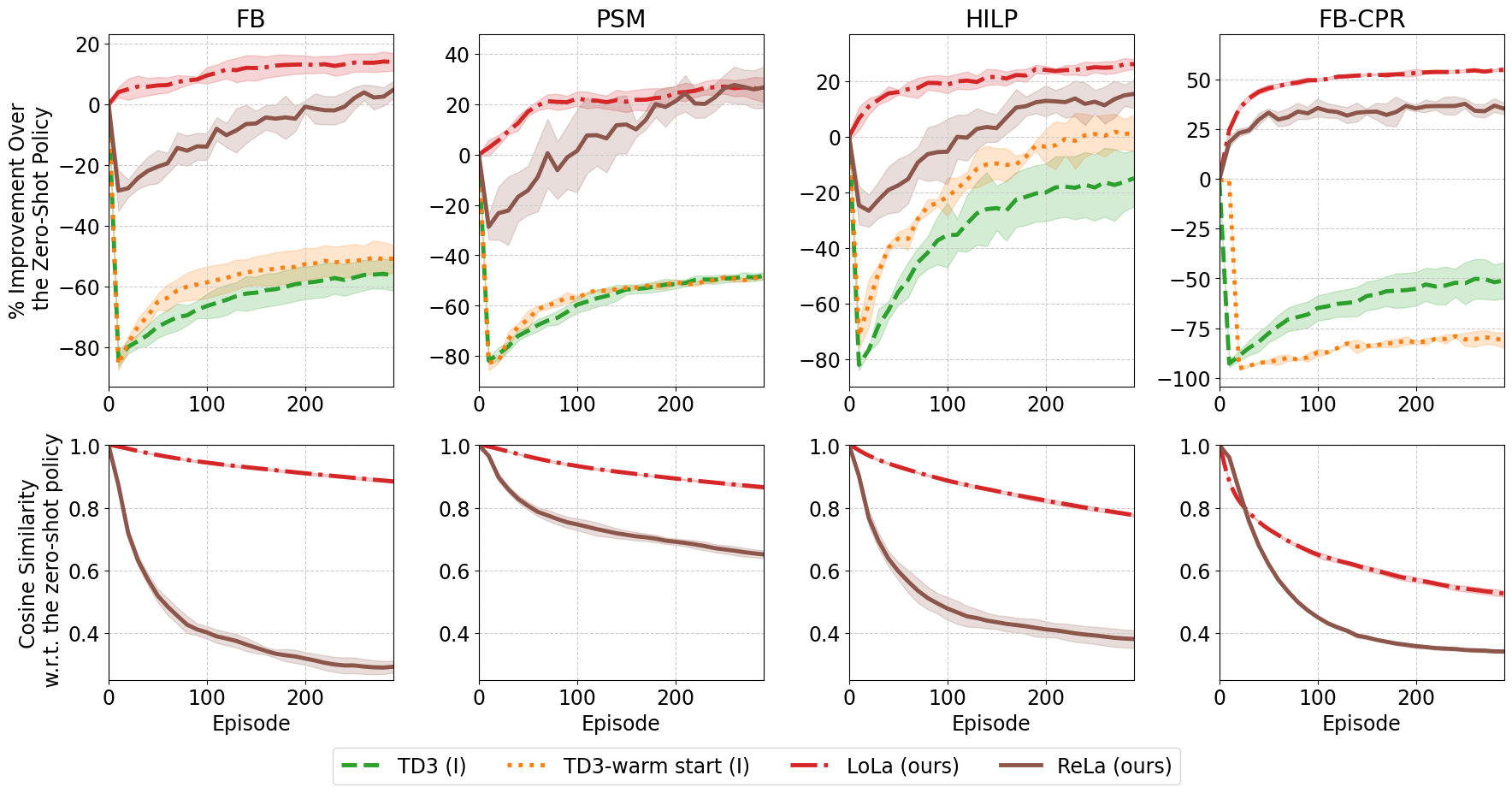}
    \caption{\textbf{Top}: Performance improvement w.r.t.\ the zero-shot policy for different online fast adaptation methods and BFMs. TD3(I) denotes standard action-based TD3 with zero-shot policy initialization, our methods are as described in Section~\ref{sec:methods}. \textbf{Bottom}: Cosine similarity between the zero-shot policy $z_{r}$ and the learned policy $z$ for the methods working in the latent policy space. We report mean and standard deviation over 5 seeds. Results are averaged over 19 tasks for FB, PSM, HILP and 45 tasks for FB-CPR.}
    \label{fig:main_result}
    \vspace{-0.3cm}
\end{figure}

Figure~\ref{fig:main_result} (\emph{top}) shows our aggregated results across tasks for each domain: Latent policy adaptation leads to performance improvements w.r.t.\ the zero-shot policy in the range of 10-30\% for DMC environment and 40-50\% for HumEnv. Compared to Figure~\ref{fig:main_summary_experiment}, these results show that significant improvements are already obtained in few online episodes. For example, \LOLA{} leads to about 10\% (resp. 40\%) improvement for DMC (resp. for HumEnv) in only 20 episodes. These results show that \emph{i)} the space of policies learned by the BFMs contains better policies than the one inferred by the zero-shot procedure and \emph{ii)} such a space can be easily navigated using gradient-based approaches. While both \RELA{} and \LOLA{} provide significant performance improvements, \LOLA{} is the only method to achieve monotonic performance improvement across the board. As we can see from the per-task visualization in Appendix~\ref{app:additional.exp}, the non-monotonic performance of \RELA{} is mostly due to the fact that the methods incurs a noticeable catastrophic forgetting in the \texttt{pointmass} environment where TD3-based methods seem to struggle in the online setting, probably due to exploration issues. As a result of training purely on online samples, critic learning in \RELA{} undergoes distribution shift which has been investigated to lead to initial unlearning~\citep{zhou2024efficient} whereas \LOLA{} skips the critic learning step entirely. In addition, \LOLA{} exploits a privileged information compared to \RELA, the ability to reset the environment to any arbitrary state, which further contributes in stabilizing and speeding up the learning process~\citep[see e.g.,][]{MhammediFR24}.

\textbf{How does the adapted policy evolve in latent space?} To try to better understand the learning dynamics of \RELA{} and \LOLA{} we report the cosine similarity between the adapted $z$ and the zero-shot policy $z_{r}$ in Figure~\ref{fig:main_result} (\emph{bottom}). \RELA{} deviates much more in the latent space from the initial zero-shot policy than \LOLA. This fast and significant change is associated with the drop in performance. On the other hand, despite the high learning rate (we found $0.1$ or $0.05$ to be the best based on the BFM), \LOLA{} remains closer to the zero-shot policy. A potential cause for the significant change in \RELA{} may be difficulties in critic learning associated with distribution shift, which can impact policy directly.  This visualization also shows that while converging to different policies, the performance of \RELA{} and \LOLA{} is comparable after $300$ episodes in the DMC environments. This reveals that policies with similar performance may be associated to with different latent vectors $z$. 

When looking at the baselines, we can notice that all action-space adaptation algorithms suffer a much more significant drop compared to latent policy adaptation. The performance gap between action-based and latent policy adaptation becomes even larger when looking at FB-CPR. In this case, all action-based algorithms completely unlearn in a few steps and are not able to rapidly recover. We think this is due to the large dimensionality of observation space, action space and policy model that lead to a much more complicated optimization problem.\footnote{A way to address this problem may be through policy regularization but this is outside the scope of this paper.} On the other hand, in contrast to other BFMs, FB-CPR does not suffer any initial performance drop when using \RELA. Indeed, all latent policy adaptation algorithms (see Appendix~\ref{app:additional.exp} for additional experiments) achieve monotonic performance improvement, stressing even more that structured search in the latent policy space may be simpler than finetuning the whole policy in high-dimensional problems.
This may be due to the fact that FB-CPR is the only BFM that is pre-trained with online environmental interactions, a setting that may reduce the distribution shift between pretraining and adaptation. Finally, the performance improvement due to the latent policy adaption is much more significant in this domain. The reason may reside in the critic training objective of FB-CPR; indeed FB-CPR uses a discriminator-based loss to regularize the policy space towards expert demonstrations. This may prevent the zero-shot inference to correctly identify the best policy for the task, while online adaptation seems to better search the policy space.

Finally, we would like to report an observation about the computational efficiency. On our hardware, LoLA runs at $\approx 157$x the FPS of ReLA and other adaptation approaches. Specifically, ReLA runs at $\approx 14$ FPS, and LoLA runs at $\approx2,200$. This gaps presumably comes from the fact that \RELA{} needs to backpropagate gradient through the BFM estimated value function and policy both in the critic and actor updates, while \LOLA{} has just a single actor update.
The computational efficiency of LoLA along with its observed near-monotonic improvement for adaptation makes it appealing in practice. 

\begin{figure}[t]
    \centering
    \includegraphics[width=1.0\linewidth]{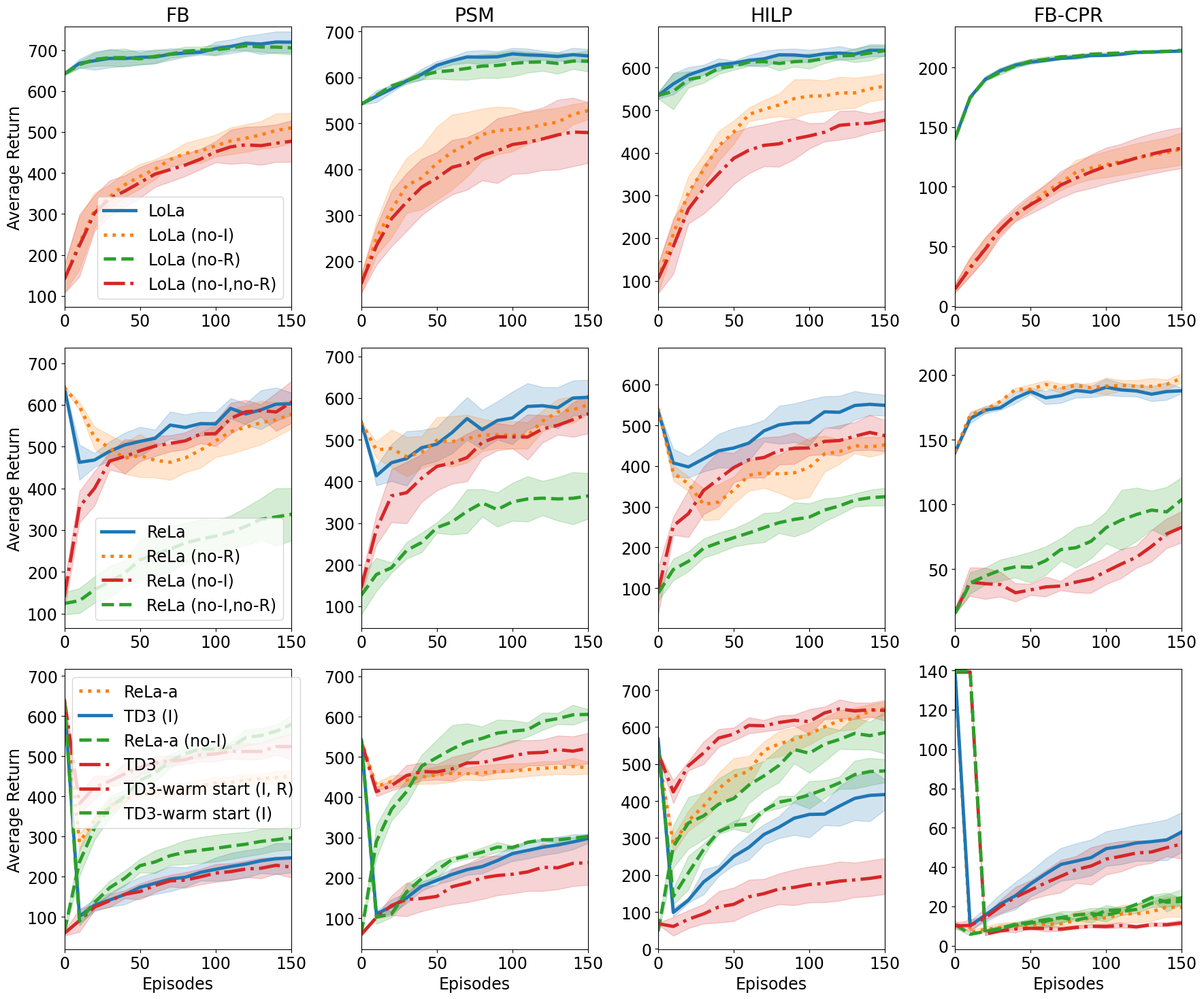}
    \caption{Average returns for several variations of \LOLA, \RELA, and action-based TD3 with warm start. We use \texttt{no-R} to denote that we do not use the BFM's estimated value function (i.e., for \LOLA{} we do not bootstrap the terminal state and for \RELA{} we learn a critic from scratch) and \texttt{no-I} to denote that we do not use zero-shot policy initialization. Finally, for TD3 we use \texttt{R} to denote that we use residual critic since the standard implementation learns a critic from scratch.}
    \label{fig:main_ablation}
    \vspace{-0.3cm}
\end{figure}

\subsection{What components are critical for fast adaptation?}\label{subsec:ablation}
\begin{wraptable}{r}{0.6\textwidth}
\centering
\scriptsize
\setlength{\tabcolsep}{2pt} %
\begin{tabular}{@{}l c c c c@{}}
\toprule
Algorithm & \makecell{Zero-Shot\\Policy Init.} & \makecell{Residual Critic($^\dagger$)/\\Bootstraped\\ Return($^+$)} & \makecell{Critic\\Trained from\\scratch} & Search space \\
\midrule
\LOLA{} & \checkmark & \checkmark ($^+$) &  & $z$ \\
\LOLA{} (no-I) &  & \checkmark ($^+$) &  & $z$ \\
\LOLA{} (no-R) & \checkmark &  & \checkmark & $z$ \\
\LOLA{} (no-I, no-R) &  &  & \checkmark & $z$ \\
\midrule
\RELA{} & \checkmark & \checkmark ($^\dagger$) &  & $z$ \\
\RELA{} (no-I) &  & \checkmark ($^\dagger$) &  & $z$ \\
\RELA{} (no-R) & \checkmark &  & \checkmark & $z$ \\
\RELA{} (no-I, no-R) &  &  & \checkmark & $z$ \\
\midrule
TD3-z &  &  & \checkmark & $z$ \\
TD3 (I) & \checkmark &  & \checkmark & $a$ \\
TD3-warm-start(I) & \checkmark &  & \checkmark & $a$ \\
TD3-warm-start(I, R) & \checkmark & \checkmark ($^\dagger$) &  & $a$ \\
\bottomrule
\end{tabular}
\caption{Summary of algorithm variations. Here, search space $z$ indicates latent policy adaptation via the policy space $\{\pi_z\}$ constructed by the BFM, while $a$ denotes fine-tuning in action space.}
\vspace{-0.3cm}
\label{tab:main.algorithms}
\end{wraptable}

In this section we assess the importance of leveraging BFM properties for fast online adaptation. We focus on ablating the need of \emph{i)} zero-shot initialization and \emph{ii)} BFM value function estimate, i.e., using a residual critic for \RELA{} and the bootstrapped Q-function for \LOLA. We focus on the very early steps of the training to better inspect the results. Ablation variants are concisely shown in Table~\ref{tab:main.algorithms} for reference, and results are reported in Figure~\ref{fig:main_ablation}.

When \emph{zero-shot initialization} is disabled  not only the performance starts lower but also take significantly longer to match the baseline's returns (if they match at all). Unsurprisingly, zero-shot initialization helps in the search process. Leveraging the BFM's value function estimate does not hurt and often helps in reducing the initial performance drop. Looking at \LOLA, BFM bootstrapping helps only marginally in all the domains. We believe that this is due to the small discount factor and large lookahead (we use $0.98$ and a lookahead of $100$ or $250$); this combination significantly reduces the role of the bootstrapped Q-function (discounted by 0.13 or 0.006). When looking at \RELA, residual critic helps in DMC domains but not in the HumEnv, where zero-shot initialization is the most important dimension. On the other hand, when zero-shot initialization is disabled in the DMC domains, the importance of residual critic is particularly evident and leads to almost match the performance of the best algorithm. Finally, the residual critic is also very important when performing direct adaptation in the action-space and helps in a faster recovery from the initial drop.

\begin{figure}[tb]
    \centering
    \includegraphics[width=0.99\textwidth]{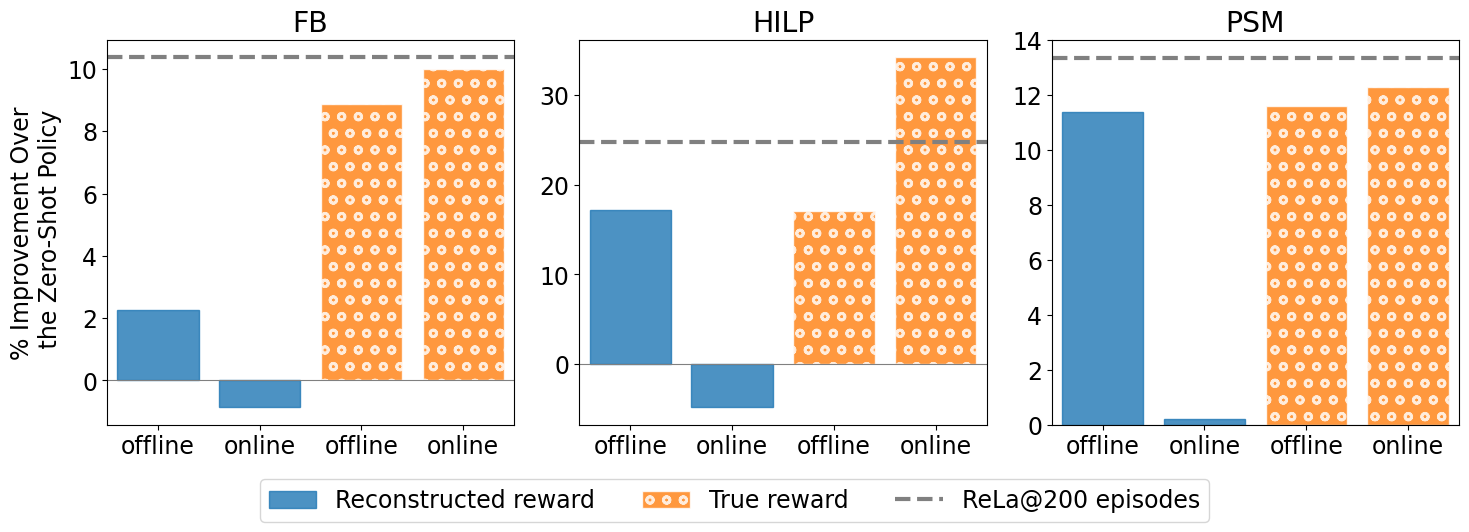}
    \caption{Performance improvement w.r.t.\ the zero-shot policy for a TD3-based method trained from scratch for 3M steps to perform search in the latent policy space (i.e., TD3-z). We report the results for both online and offline training using the ExoRL~\citep{exorl} dataset. We also ablate learning with the true task reward and the reward reconstructed by the BFM methods. We average the results over all the task of the \texttt{Walker}, \texttt{Quadruped} and \texttt{Cheetah} domains. We report the average performance over 5 seeds. We additionally show the performance of \RELA~ after training for $200$ episodes.}
    \label{fig:td3_ablation}
\end{figure}

\subsection{Dissecting the Suboptimality of Zero-Shot RL Policies}
The previous results show that BFMs are indeed learning skills that contain good policies for all the downstream tasks we study. This then raises the question on what causes the suboptimality of the zero-shot policy and the need of performing online adaptation to actually recover a better policy. We run a series of ablations with FB, HILP and PSM. We do not consider FB-CPR because our ablation involves offline training and we do not have access to an offline dataset for this model since it was trained online. We consider TD3 as learning algorithm since it is the building block of all the three BFMs and focus on latent policy adaptation (we call this approach TD3-z to avoid confusion with TD3 used in the previous sections to optimize the full policy network). For all experiments in this section, we consider the standard scenario of \textbf{training from scratch}, no zero-shot initialization and no-residual critic. Specifically when searching in $z$ space, we use a pretrained BFM actor and initialize $z$ along with the critic randomly and when learning in action space we initialize both the actor and critic randomly.  We report the performance of TD3-z after 3M training steps when using the true reward function and the reward function reconstructed by the BFMs\footnote{Latent or reconstructed reward is given by $\tilde{r}_z(s)=\varphi(s).z_r$}, both offline and online. We do not consider \texttt{pointmass} in this test since TD3 does not work well when trained online due to the challenging exploration in the long-horizon tasks considered in this domain.

Overall, these experiments confirm that BFMs can express much better policies than the zero-shot policy and that optimizing for true rewards is crucial to unlock their full performance. When optimizing for the latent reward, offline TD3-z can already improve the zero-shot performance revealing the difficulty of optimizing all polices $\{\pi_z\}$ simultaneously during the pre-training process.\footnote{The fact that performance of HILP and PSM improve by about than 10-15\% by offline training on the reconstructed reward might may be due to a non-perfect pre-train. Indeed, the pre-training condition should ensure that the actor is already optimal on any reconstructed reward on the training data distribution.} Interestingly, when moving to online training on the latent reward performance can even drop. We conjecture the cause is the distribution shift between online and offline samples. Given that the models were trained offline, their reward prediction degrades on out-of-distribution samples encountered during online adaptation, further skewing towards learning policies that are even less correlated to the true reward. This is confirmed when looking at the performance when optimizing for the true reward, which consistently lead to better results across offline and online tests, with online methods being overall better. This ablation confirms that focusing on searching in the z-space, while correcting the embedding errors is the right strategy to achieve fast adaptation online. Indeed, we see that \RELA{}  recovers better policies than the one obtained by training from scratch TD3-z online for 3M episodes in only 200 episodes. Even faster if we use \LOLA. This shows that leveraging information from the BFM is useful in many cases.

\section{Related Work}

    \textbf{Unsupervised RL pre-training:} For language and vision, unsupervised pretraining has paved the way to extracting meaningful structure from data, scaling up, and obtaining impressive results for transfer and zero-shot generalization to different downstream tasks. In recent years, approaches have been proposed for unsupervised reinforcement learning: a training paradigm where a learning agent attempts to extract world structure and representations that will later allow it to solve diverse multi-step decision-making problems in the environment. Various objectives have been proposed: world modeling~\citep{bruce2024genie,hansen2023td}, intrinsic rewards~\citep{schmidhuber2019reinforcement,stadie2015incentivizing,sekar2020planning,pathak2017curiosity}, empowerment and mutual information skill learning~\citep{klyubin2005empowerment,eysenbach2018diversity, rajeswar2023mastering,gregor2016variational}, goal-reaching~\citep{ma2022vip,park2023metra,park2024foundation}, successor measures~\citep{fballreward,agarwal2024proto,sikchi2024rl} among others. In this work, we restrict our focus to the class of unsupervised RL objectives that learn a family of policies and allow us to query for a near-optimal policy given any test-time reward function without further learning or planning in the environment. Approaches belonging to this class often learn a state representation and use that to define the class of reward functions for which they learn the set of optimal policies. At inference time, they output the policy that is optimal for the projection of the reward function to this class of reward functions.

\textbf{Fine-tuning and adaptation with unsupervised RL models:} Similar to supervised pre-training approaches, unsupervised zero-shot RL models are not expected to output optimal policies for the given task. Rather they are expected to output a reasonable policy initialization that can be later finetuned or adapted. Prior approaches for policy adaptation with pre-trained RL models have mostly studied the offline-to-online setting. In this setting, a policy is trained with reward-labeled transitions first using offline data with specialized offline RL algorithms and then allowed to fine-tune by interacting with the environment and \textit{retaining access to the offline data}. Offline RL algorithms~\citep{levine2020offline,sikchi2023dual} incorporate pessimism to avoid overestimation by restricting the policy to visit states closer to the dataset. Naively using the same algorithm to finetune online has been observed to lead to slow performance improvements and using an online RL algorithm leads to performance collapse at the beginning of fine-tuning~\citep{luo2023finetuning}. This behavior has been attributed to a distribution shift for critic-learning~\citep{pmlr-v202-yu23k}, and prior works have investigated various techniques, such as calibration of Q-functions to mitigate this problem~\citep{nakamoto2023cal}. Our work, considers a different but practical paradigm for adaptation where a) no offline data is retained during finetuning and b) we learn from reward-free transitions. First, by only retaining pre-trained models we reduce compute requirements of learning from large pre-training offline datasets~\citep{zhou2024efficient}, and by considering reward-free transitions we have a single model that can adapt to any downstream task~\citep{kim2024unsupervised}. Learning online without retaining offline data suffers a significant initial drop of performance with respect to the pre-trained policy as recently investigated by \citet{zhou2024efficient}.

\section{Conclusion}
Unsupervised zero-shot RL pre-training can result in an agent (a type of Behavioral Foundation Model, BFM) capable of accomplishing a wide variety of tasks having a noticeable but expected degree of suboptimality. This paper investigates and addresses the question of how to adapt these agents to be better at a task specified during test-time with limited environment interactions. We propose two fast adaptation strategies (\LOLA{} and \RELA{}); The key insight behind our methods is to reuse pre-trained knowledge from BFM strategically and search over the learned latent policy space that provides a low-dimensional landscape favorable for gradient-based optimization. We have demonstrated the effectiveness of these strategies across various zero-shot BFMs. Notably, \LOLA{}, an actor-only adaptation algorithm, demonstrates monotonic performance improvement on all domains and BFMs, making it a reliable choice when privileged resets are permitted. However, our findings also reveal an initial performance drop when employing any actor-critic method, including our proposed \RELA{} algorithm. This highlights the need for further investigation into mitigating forgetting in the actor-critic class of approaches. Future research directions include exploring meta-learning adaptation techniques, including in-context adaptation by learning to adapt in multi-task settings to optimize learning costs and improve overall performance. 

\clearpage
\newpage
\bibliographystyle{assets/plainnat}
\bibliography{paper}

\clearpage
\newpage

\beginappendix

\section{Motivation for residual learning}
\label{ap:residual_motivation}
\paragraph{Notation: } 
We use matrix form and we identify for any $z \in \mathcal{Z}$, $\psi_z: \mathcal{S} \times \mathcal{A} \rightarrow \R^d$ and $\phi: \mathcal{S} \times \mathcal{A} \rightarrow \R^d$ by $\psi_z \in \R^{d \times |\mathcal{S} \times \mathcal{A}|}$ and $\phi \in \R^{d \times |\mathcal{S}|}$ respectively and $D_\rho = \text{diag}((\rho(s))_{s \in \mathcal{S}})$. Similarly, we identify for any $z \in \mathcal{Z}$, $F_z: \mathcal{S} \times \mathcal{A} \rightarrow \R^d$ and $B: \mathcal{S} \times \mathcal{A} \rightarrow \R^d$ by $F_z \in \R^{d \times |\mathcal{S} \times \mathcal{A}|}$ and $B \in \R^{d \times |\mathcal{S}|}$ respectively.
We denote by $\Pi_{\phi} = \phi^\top \left( \phi D_\rho \phi^\top \right)^{-1} \phi D_\rho$ the $L^2(\rho)$ orthogonal projection onto the linear span of $\phi$.
\begin{proposition}
    Let $\phi: S \rightarrow \R^d$ a state feature map and $\{\psi_z \}_{z \in Z}$ the corresponding universal successor features for the policy family $\{\pi_z \}_{z \in Z}$, \textit{i.e} 
        $\psi_z(s, a) = \mathbb{E} [ \sum_{t \geq 0} \gamma^t \phi(s_{t+1}) \mid (s, a), \pi_z ] $
Then, for any reward function $r: S \rightarrow R$, we have: 
$Q_r^{\pi_z} = \psi_z^\top z_r + Q^{\pi_z}_{ r - \phi^\top z_r}$
where $z_r = \mathbb{E}_{s \sim \rho }[\phi(s) \phi(s)^\top]^{-1} \mathbb{E}_{s \sim \rho }[\phi(s) r(s)]$, and $Q^{\pi_z}_{r - \phi^\top z_r }$ is the Q-function of the residual reward $r - \phi^\top z_r = ( I - \Pi_{\phi}) r$.
\end{proposition}

\begin{proof}

    for any reward function $r \in \R^{S}$, we have $r = ( \Pi_{B} + I - \Pi_{B})r$, then 
    \begin{align*}
        Q^{\pi_z}_r & = M^{\pi_z} r \\
        & =  M^{\pi_z} ( \Pi_{B} + I - \Pi_{B})r  \\ 
        & = M^{\pi_z} \Pi_{B} r + M^{\pi_z} (I - \Pi_{B})r \\
        & = M^{\pi_z} \phi^\top \left( \phi D_\rho \phi^\top \right)^{-1} \phi D_\rho r + M^{\pi_z} (r - \phi^\top \left( \phi D_\rho \phi^\top \right)^{-1} \phi D_\rho r ) \\
        & = \psi_z^\top z_r + Q^{\pi_z}_{r - \phi^\top z_r} 
    \end{align*}
    where the last equation follows from the fact that $\psi_z^\top = M^{\pi_z} \phi^\top $ and $z_r = \mathbb{E}_{s \sim \rho }[\phi(s) \phi(s)^\top]^{-1} \mathbb{E}_{s \sim \rho }[\phi(s) r(s)] = \left( \phi D_\rho \phi^\top \right)^{-1} \phi D_\rho r$
\end{proof}
\begin{proposition}
    Let assume that for any $z \in Z$,  $F_z$ is a stationary point of the FB training loss $\ell(F, B)$, namely, the functional derivative $\frac{\partial l}{\partial F_z}$ of the loss with respect of $F_z$ is 0.
    Then, 
    \begin{align*}
    Q^{\pi_z}_r & = F_z^\top z_r + Q^{\pi_z}_{(I - \Pi_{B})r} \\
    & = F_z^\top z_r + 
    \left(M^{\pi_z} - F_z^\top B D_\rho \right) \left(r - \Pi_{B} r \right)
\end{align*}
where $z_r = \mathbf{E}_{s \sim \rho}[B(s) r(s)]$.
\end{proposition}
\begin{proof}
Let's remind the FB training loss:
\begin{equation*}
    \ell(F, B) = \E_{\substack{z, (s, a) \sim \rho \\ s^+ \sim \rho }} \left[ \left( F(s, a, z)^\top B(s^+) - P(\mathrm{d}s^+ \mid s, a) / \rho(\mathrm{d}s^+) -  (P^{\pi_z} \bar{F})(s, a, z)^\top \bar{B}(s^+)\right)^2\right]
\end{equation*}
In matrix form, we obtain: 
\begin{equation*}
    \ell(F, B) = \E_z \left[ \text{Trace}\left( \left( F_z^\top B - PD_\rho^{-1} - \gamma P^{\pi_z} \bar{F}_z^\top \bar{B} \right)^\top D_\rho \left( F_z^\top B - PD_\rho^{-1} - \gamma P^{\pi_z} \bar{F}_z^\top \bar{B} \right) D_\rho\right) \right]
\end{equation*}
if $F_z$ satisfies the  stationarity conditions, \textit{i.e}, $\frac{\partial \ell}{\partial F_z} = 0$, then, we have 
\begin{align*}
    \frac{\partial \ell}{\partial F_z} = 0 & \Rightarrow 2 B D_\rho \left( F_z^\top B - PD_\rho^{-1} - \gamma P^{\pi_z} F_z^\top B \right)^\top D_\rho = 0 \\
    & \Rightarrow 2 D_\rho \left( F_z^\top B - PD_\rho^{-1} - \gamma P^{\pi_z} F_z^\top B \right) D_\rho B^\top = 0 \\
    & \Rightarrow F_z^\top B D_\rho B^\top = PB^\top + \gamma P^{\pi_z}F_z^\top B D_\rho B^\top \\
    & \Rightarrow F_z^\top = M^{\pi_z} B^\top \left( B D_\rho B^\top \right)^{-1} \\
    & \Rightarrow F_z^\top B D_\rho = M^{\pi_z} B^\top \left( B D_\rho B^\top \right)^{-1} \\
    & \Rightarrow F_z^\top B D_\rho = M^{\pi_z} B^\top \left( B D_\rho B^\top \right)^{-1} B D_\rho
\end{align*}
Therefore  $F_z^\top B D_\rho = M^{\pi_z} \Pi_{B}$ where $\Pi_{B} = B^\top \left( B D_\rho B^\top \right)^{-1} B D_\rho$ is the $L^2(\rho)$ orthogonal projection onto the linear span of $B$.

Let $\Pi_{B^\perp}$ the orthogonal projection onto the orthogonal of $B$. By definition, we have $\Pi_{B^\perp} = I - \Pi_{B}$

We have: 
\begin{align*}
    Q^{\pi_z}_r & = M^{\pi_z} r \\
    & = M^{\pi_z} (\Pi_B + \Pi_{B^\perp}) r \\
    & = M^{\pi_z} \Pi_B r + M^{\pi_z} \Pi_{B^\perp} r \\
    & = F_z^\top B D_\rho r + M^{\pi_z} \Pi_{B^\perp} r \\
    & = F_z^\top z_r + M^{\pi_z} \Pi_{B^\perp} r
\end{align*}
where $z_r = B D_\rho r = \E_{s \sim \rho}[B(s) r(s)]$

Therefore, we have:
\begin{align*}
    Q^{\pi_z}_r = F_z^\top z_r + Q^{\pi_z}_{\Pi_{B^\perp} r}
\end{align*}
where the second term is the the Q-function of the residual reward $\Pi_{B^\perp} r$ 

Moreover, since $\Pi_{B^\perp}^2 = \Pi_{B^\perp}$, we can write: 
\begin{align*}
    Q^{\pi_z}_r & = F_z^\top z_r + 
    \left(M^{\pi_z} \Pi_{B^\perp} \right) \left(\Pi_{B^\perp} r \right) \\
    & = F_z^\top z_r + 
    \left(M^{\pi_z} - F_z^\top B D_\rho \right) \left(r - \Pi_{B} r \right)
\end{align*}
Which means that the residual term can be expressed as the successor measure approximation error (due to the low-rank decomposition of FB model)  multiplied by the reward error (due to the reward embedding in the span of B).
 \end{proof}

\section{Pseudocode}
Algorithm~\ref{algo:rela} and~\ref{algo:lola} outline pseudocode for ReLA and LoLA respectively.
\begin{figure}[h]
    \centering
    \begin{minipage}{0.47\textwidth}
        \begin{algorithm}[H]
            \scriptsize
        \label{algo:rela}
            \caption{ReLA}
            \SetAlgoLined
            \DontPrintSemicolon
            \textbf{Load} Frozen BFM's successor features $\psi(s,a,z)$ and policy $\pi_z(s)$ networks.\;
            \textbf{Initialize} residual critic networks $Q^{\text{residual}}_{\theta_1}$, $Q^{\text{residual}}_{\theta_2}$, replay buffer $\mathcal{D}_{\text{online}}$, exploration std $\sigma$, Update to Data ratio (UTD) $M$, Initialize target networks: $Q^{\text{residual}}_{\theta'_1} \leftarrow Q^{\text{residual}}_{\theta_1}$, $Q^{\text{residual}}_{\theta'_2} \leftarrow Q^{\text{residual}}_{\theta_2}$.\;
            \textbf{Compute zero-shot latent $z_{r}$} using inference samples for the BFM agent with test-time reward function.\;
            \For{each environment step $t$}{
                Select $a_t = \pi_{z}(s_t) + \epsilon$,  $\epsilon \sim \mathcal{N}(0, \sigma)$\;
                Execute $a_t$; observe $r_t$, $s_{t+1}$\;
                Store $(s_t, a_t, r_t, s_{t+1})$ in $\mathcal{D}_{\text{online}}$\;
                Sample M mini-batches $\text{Batch}_i=\{(s_i, a_i, r_i, s'_i)\}\sim \mathcal{D}_{online}$\;
                Compute target Q-value: $ y_i = r_i + \gamma (\psi(s'_i,\pi_z(s'_i),z_r)\cdot z_r+\min\{Q_{\theta'_1}(s'_i, \pi_z(s'_i)),\, Q_{\theta'_2}(s'_i, \pi_z(s'_i))\})$\;
                $\textbf{TemporalDifferenceUpdate}(\psi(s_i,a_i,z_r)\cdot z_r+Q^{\text{residual}}_{\theta_{k\in[1,2]}},y_i)$ for $i \in[M]$ using critic parameterization from Eq~\ref{eq:residual_decomposition}\;
                {\textbf{Latent Policy Update} Update $z$ taking gradient step as in Eq~\ref{eq:latent_policy_update} on $\cup_{i \in[m]}\text{Batch}_i$ .}\;
                Update target networks by polyak averaging;
            }
        \end{algorithm}
    \end{minipage}
    \hfill
    \begin{minipage}{0.47\textwidth}
        \begin{algorithm}[H]
            \scriptsize
        \label{algo:lola}
            \caption{LoLA}
            \SetAlgoLined
            \DontPrintSemicolon
            \textbf{Load} Frozen BFM's successor features $\psi(s,a,z)$ and policy $\pi_z(s)$ networks\;
            \textbf{Initialize} latent policy $\pi_{\mu, \sigma}=\mathcal{N}(\mu=z_{r},\sigma)$, replay buffer $\mathcal{D}_{\text{online}}$, sampling state distribution $\nu(\mathcal{D}_{\text{online}},d_0)$, z budget $k$, intial state budget $m$, horizon $n$\;
            \textbf{Compute zero-shot latent $z_{r}$} using inference samples for the BFM agent with test-time reward function.\;
            \For{each gradient step}{
            \For{b=1..m}{
                $s_0 \sim \mu(\mathcal{D}_{\text{online}},d_0)$ \;
                \For{i=1..k}{
                    $z^i_{b} \sim \pi_{\mu, \sigma}$, Reset to $s_0$  \;
                    Rollout trajectory $\tau^i_{b}$ by taking actions given by $a_t = \pi_{z^i_{b}}(s_t)$ \;
                    Compute $R(s_0,z^i_{b})$\;
                    Collect states from $\tau^i_{b}$ in $\mathcal{D}_{\text{online}}$
                }
                
            }
            Update $\pi_{\mu, \sigma}$ by taking gradient step in Eq~\ref{eq:lola_loss}.
            }
        \end{algorithm}
    \end{minipage}
    \caption{Pseudocode of our proposed adaptation methods: Residual Latent Adaptation (ReLA) and Lookahead Latent Adaptation (LoLA).}
    \vspace{-0.3cm}
\end{figure}

\section{Experimental Setup}
\label{ap:experiments}
\subsection{Environments}
\label{ap:env_info}

We list the continuous control environments from the DeepMind Control Suite \citep{tassa2018deepmind} and Humenv~\citep{tirinzoni2025zeroshot} used in this work in Table~\ref{tab:env_list}.

\begin{table}[h]
    \centering
    \begin{tabular}{cccc}
        \hline
        Domain & Observation dimension & Action dimension & Episode length\\
        \hline
        Pointmass &4 &2 & 1000\\
        Walker & 24 & 6 & 1000\\
        Cheetah & 17 & 6 & 1000\\
        Quadruped & 78 & 12 & 1000\\
        HumEnv & 358 & 69 & 300\\
        \hline
    \end{tabular}
    \caption{Overview of observation spaces, action spaces and episode length of environments used in this work.}
\label{tab:env_list}
\end{table}

\subsection{Behavioral Foundation Models}
We trained all the BFMs except for the FB-CPR model that is publicly available (\href{https://github.com/facebookresearch/metamotivo}{code link}).

\paragraph{Offline BFMs.} We train the BFMs using the publicly available dataset from ExoRL~\citep{exorl} collected using the RND algorithm~\citep{DBLP:conf/iclr/BurdaESK19}. We used the authors implementation for FB and FB-CPR (\href{https://github.com/facebookresearch/metamotivo}{code link}) and reimplemented PSM and HILP. We report in Table~\ref{table:bfm-hparams} the set of hyperparameter used for the algorithms.

\begin{table}[h]
\renewcommand{\arraystretch}{1.25}
\centering
\caption{BFM hyperparameters. We largely reuse the hyperparameters from \citet{pirotta2024fast} for FB, from~\citep{park2024foundation} for HILP.}\label{table:bfm-hparams}
\resizebox{
\ifdim\width>\columnwidth
    \columnwidth
  \else
    \width
  \fi
}{!}{
\begin{tabular}{@{}cclllll@{}}
\toprule
& & {\fontsize{11}{13}\selectfont \textbf{Hyperparameter}}
& {\fontsize{11}{13}\selectfont \textbf{Walker}}
& {\fontsize{11}{13}\selectfont \textbf{Cheetah}}
& {\fontsize{11}{13}\selectfont \textbf{Quadruped}}
& {\fontsize{11}{13}\selectfont \textbf{Pointmass}} \\ \midrule
\multirow{7}{*}{FB} &\multirow{5}{*}{\makecell[c]{Forward Backward\\ \citep{fballreward}}} &
Embedding Dimension $d$ & $100$ & $50$ & $50$ & $100$ \\
&& Embedding Prior & $S^d$ & $S^d$ & $S^d$ & $S^d$ \\
&& Embedding Prior Goal Prob.& $0.5$ & $0.5$ & $0.5$ & $0.5$ \\
&& $B$ Normalization & $\ell_2$ & {\,$\ell_2$\,} & $\ell_2$ & {\,$\ell_2$\,} \\
&& \makecell[l]{Orthonormal Loss Coeff.} & $1$ & $1$ & $1$ & $1$ \\ \cline{2-7}
&\multirow{3}{*}{\makecell[c]{Optimizer (Adam)\\ \citep{kingma15adam}}} 
& Learning Rate (F, B)  & ($10^{-4}$, $10^{-4}$) & ($10^{-4}$, $10^{-4}$) & ($10^{-4}$, $10^{-4}$) & ($10^{-4}$, $10^{-6}$) \\
&& Learning Rate ($\pi$)  & $10^{-4}$ & $10^{-4}$ & $10^{-4}$ & $10^{-6}$ \\
&& Target Network EMA & $0.99$ & $0.99$ & $0.99$ & $0.99$ \\
\midrule\midrule
\multirow{3}{*}{HILP} & \multirow{1}{*}{\makecell[c]{Hilbert Representations\\ \citep{park2024foundation}}}  
& Embedding Dimension $d$ & $50$ & $50$ & $50$ & $100$ \\
&& Feature Learning Expectile  &$0.5$ &$0.5$ &$0.5$ & $0.5$ \\
&& Feature Learning Discount Factor  & $0.98$ &$0.98$ &$0.98$ &$0.98$ \\
&&Successor feature loss & Q-loss &Q-loss &Q-loss & Q-loss\\
\cline{2-7}
& \multirow{2}{*}{\makecell[c]{Optimizer (Adam)}} & Learning Rate (SF, F)  & ($10^{-4}$, $10^{-5}$) & ($10^{-4}$, $10^{-4}$) & ($10^{-4}$, $10^{-4}$) & ($10^{-4}$, $10^{-4}$) \\
&& Learning Rate ($\pi$)  & $10^{-4}$ & $10^{-4}$ & $10^{-4}$ & $10^{-4}$ \\ 
&& Target Network EMA features & $0.995$ & $0.995$ & $0.995$ & $0.995$ \\
&& Target Network EMA SF & $0.99$ & $0.99$ & $0.99$ & $0.99$ \\
\midrule\midrule
\multirow{7}{*}{PSM} &\multirow{5}{*}{\makecell[c]{Proto Successor Measures\\ \citep{agarwal2024proto}}} &
Embedding Dimension $d$ & $100$ & $50$ & $50$ & $100$ \\
&& Policy Codebook Size  & $2^{16}$ & $2^{16}$ & $2^{16}$ & $2^{16}$ \\
&& Feature Learning Timesteps& $400k$ & $400k$ & $400k$ & $400k$ \\
&& Embedding Prior Goal Prob.& $0.5$ & $0.5$ & $0.5$ & $0.5$ \\
&& $B$ Normalization & $\ell_2$ & {\,$\ell_2$\,} & $\ell_2$ & {\,$\ell_2$\,} \\
&& \makecell[l]{Orthonormal Loss Coeff.} & $1$ & $1$ & $1$ & $1$ \\ \cline{2-7}
&\multirow{3}{*}{\makecell[c]{Optimizer (Adam)\\ \citep{kingma15adam}}} 
& Learning Rate (F, B)  & ($10^{-4}$, $10^{-4}$) & ($10^{-4}$, $10^{-4}$) & ($10^{-4}$, $10^{-4}$) & ($10^{-4}$, $10^{-6}$) \\
&& Learning Rate ($\pi$)  & $10^{-4}$ & $10^{-4}$ & $10^{-4}$ & $10^{-6}$ \\
&& Target Network EMA & $0.99$ & $0.99$ & $0.99$ & $0.99$ \\ \midrule\midrule
& \multirow{3}{*}{\makecell[c]{Policy (TD3)\\ \citep{fujimoto2018addressing}}} & Target Policy Noise & $\mathcal{N}(0, 0.2)$ & $\mathcal{N}(0, 0.2)$ & $\mathcal{N}(0, 0.2)$ & $\mathcal{N}(0, 0.2)$ \\
& & Target Policy Clipping & $0.3$ & $0.3$ & $0.3$ & $0.3$ \\
& & Policy Update Frequency & $1$ & $1$ & $1$ & $1$ \\ \midrule
& \multirow{5}{*}{\makecell[c]{\hfill Common \hfill}} & Batch Size     & {$1024$} & $1024$ & {$1024$} & $1024$ \\
&& Gradient Steps & $3$M & $3$M & $3$M & {$3$M} \\
&& Discount Factor $\gamma$ & $0.98$ & $0.98$ & $0.98$ & $0.99$ \\
&& \makecell[l]{Reward Inference Samples} & {$250,000$}& {$250,000$}& {$250,000$}& {$250,000$}\\
&& \makecell[l]{ExoRL number of trajectories} & $5,000$& $5,000$& $5,000$& $5,000$\\
\bottomrule
\end{tabular}}
\end{table}

\paragraph{FB architecture.}
The backward representation network $B(s)$ is represented by a feedforward neural network with two hidden layers, each with $256$ units, that takes as input a state and outputs a $d$-dimensional embedding. 
For the forward network $F(s,a,z)$, we first preprocess separately $(s,a)$ and $(s,z)$ by two
feedforward networks with one single hidden layer (with $1024$ units) to $512$-dimentional space.
Then we concatenate their two outputs and pass it into two heads of feedforward networks (each
with one hidden layer of $1024$ units) to output a $d$-dimensional vector.
For the policy network $\pi(s,z)$, we first preprocess separately $s$ and $(s,z)$ by two feedforward
networks with one single hidden layer (with $1024$ units) to $512$-dimentional space. Then we
concatenate their two outputs and pass it into another one single hidden layer feedforward network
(with $1024$ units) to output to output a $d_A$-dimensional vector, then we apply a \texttt{Tanh} activation
as the action space is $[-1,1]^{d_A}$.

For all the architectures, we apply a layer normalization and \texttt{Tanh} activation in the
first layer in order to standardize the states and actions. We use \texttt{Relu} for the rest of layers. We also
pre-normalize $z:z \leftarrow \sqrt{d} \frac{z}{\|z\|_2}$ in the input of $F$, and $\pi$.

\paragraph{HILP architecture.}
We use the same policy architecture as FB as well F-architecture for the successor features. We learn the HILP features using a 2 layers MLP with hidden dimension 1024. Even in this case, $z$ is normalized.

\paragraph{PSM architecture.}

We use the same policy architecture as FB as well F-architecture for the successor features. We learn the PSM features using a 2 layers MLP with hidden dimension 256. Even in this case, $z$ is normalized.

\subsection{Algorithm Implementation}

All the actor-critic algorithms are implemented using TD3~\citep{fujimoto2018addressing} as the base offpolicy algorithm. When learning from scratch we use a 2 two layers MLP with hidden dimension $1024$ and \texttt{Relu} activation. We use the same configuration also for the critic.

For \RELA, we use a small 2 layer MLPs with 64 hidden dimensions and ReLU activation as residual network. In the ablation, when residual critic is deactivated, we use the same critic network as for standard TD3.

For \LOLA, we use a Gaussian policy centered around the learned $z$ and learn simultaneously mean and standard deviation.

\subsection{Hyperparameters}

For all baselines and our method, we run a hyperparameter sweep across domains and tasks and choose the configuration that performs the best across tasks for each BFM. %
\paragraph{TD3-based algorithms.}
We run a hyperparameter sweep on Update to Data ratio (UTD) in $[1,4,8]$, actor update in frequency in $[1,4]$. We use a small  $2$ layer MLP with 64 hidden nodes for the residual network which we found to work best for fast adaptation. When not using residual critic, we learn a critic from scratch using a  $2$ layer MLP with 1024 hidden nodes. We use $10^{-4}$ as learning rate for both critic and actor. We use either warm start of $0$ steps or $5000$ steps.

\paragraph{\LOLA.} We consider hyperparameter sweep between a lookahead horizons of $[50,100,250]$, the number of total trajectories per update to be $10$, and number of trajectories for a sampled state to be $5$ (for calculating baseline). We sweep between $[0, 0.2,0.5]$ for the probability of resetting to initial state distribution and otherwise sampling from states encountered in replay buffer. We sweep also the learning rate in $[0.1, 0.05]$.

\section{Additional Experiments}\label{app:additional.exp}

As mentioned in the main paper, \texttt{pointmass} is the domain where actor-critic algorithms incurs a significant initial drop. Figure~\ref{fig:main_result_nopointmass} shows the average performance improvement without the \texttt{pointmass} domain. As we can see, \RELA{} has still a initial drop but it is much more reduced compared to what reported in the main. Previous papers~\citep[e.g.][]{pirotta2024fast}, noticed that a smaller learning rate helped in \texttt{pointmass}. In our experiments we kept the learning rate fixed at $10^{-4}$ for all the domains, it would be interesting to test different values.

\begin{figure}[h]
    \centering
    \includegraphics[width=1.0\linewidth]{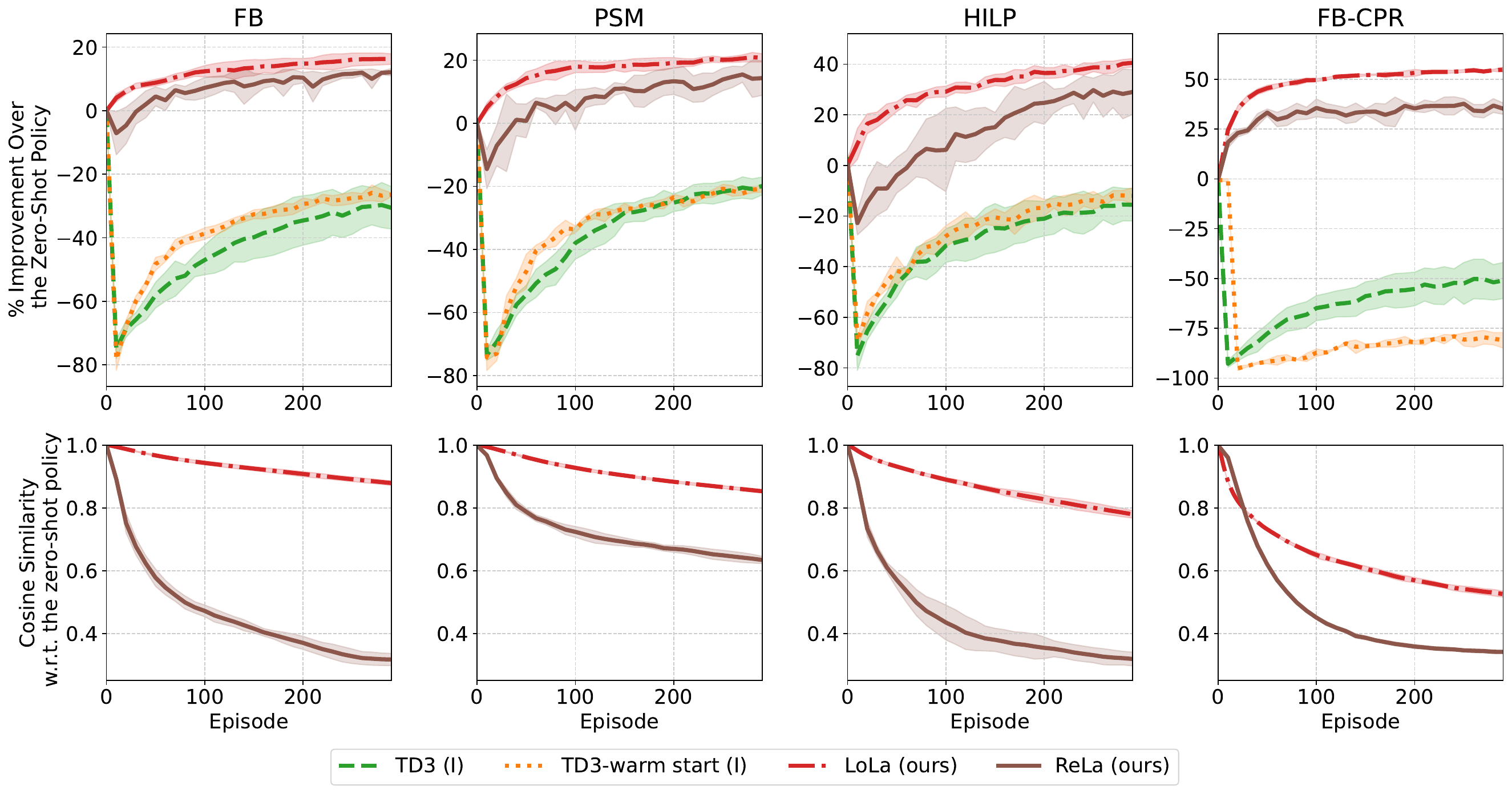}
    \caption{\textbf{Top}: Performance improvement w.r.t.\ the zero-shot policy for different online fast adaptation methods and BFMs without the \texttt{Pointmass} domain.}
    \label{fig:main_result_nopointmass}
\end{figure}

\subsection{Per Algorithm Per Domain Ablation Studies}

We conduct extensive ablation studies to understand the impact of key design choices in our methods, specifically: 
(1) zero-shot initialization in LoLA and ReLA variants, (2) value function bootstrapping in LoLA, (3) residual critics in ReLA variants and action-based TD3 with warm start. Table~\ref{tab:appendix.algorithms} provides a comprehensive list of the algorithm variants considered.

\begin{table}[h]
    \centering
    \resizebox{.96\textwidth}{!}{
    \begin{tabular}{ccccccc}
    \toprule
    Algorithm & \makecell{Zero-Shot Policy\\ Initialization} & \makecell{Residual Critic ($^\dagger$) or\\  Bootstrapped Return ($^+$)}  & \makecell{Critic Trained\\ from scratch} & \makecell{Search space} & WSRL\\
    \midrule
    \LOLA{} & \checkmark& \checkmark ($^+$) & & $z$ & & \multirow{4}{*}{actor-only}\\
    \LOLA{} (no-I) & & \checkmark ($^+$) & & $z$ \\
    \LOLA{} (no-R) & \checkmark& &\checkmark & $z$ \\
    \LOLA{} (no-I, no-R) & & &\checkmark & $z$ \\
    \midrule
    \RELA{} & \checkmark& \checkmark ($^\dagger$) & & $z$&& \multirow{4}{*}{actor-critic} \\
    \RELA{}-warm-start & \checkmark& \checkmark ($^\dagger$) & & $z$& \checkmark \\
    \RELA{} (no-I) & & \checkmark ($^\dagger$) & & $z$ \\
    \RELA{} (no-R) & \checkmark& &\checkmark & $z$ \\
    \RELA{}-warm-start (no-R) &  \checkmark& &\checkmark & $z$& \checkmark \\
    \RELA{} (no-I, no-R) & & &\checkmark & $z$ \\
    \midrule
    \RELA{}-a & \checkmark& \checkmark ($^\dagger$) & & $a$&& \multirow{4}{*}{actor-critic} \\
    \RELA{}-a (no-I) & & \checkmark ($^\dagger$) & & $a$ \\
    \RELA{}-a (no-R) & \checkmark& &\checkmark & $a$ \\
    \RELA{}-a (no-I, no-R) & & &\checkmark & $a$ \\
    \midrule
    TD3-z & & & \checkmark & $z$&& \multirow{4}{*}{actor-critic} \\
    TD3 (I) & \checkmark & & \checkmark & $a$ \\
    \makecell{TD3-warm-start (I)\\ (i.e., using WSRL)} & \checkmark & & \checkmark & $a$& \checkmark  \\
    \makecell{TD3-warm-start (I, R)} & \checkmark & \checkmark ($^\dagger$) & & $a$& \checkmark  \\
    \bottomrule
    \end{tabular}
    }
    \caption{Summary of the algorithm variations considered in the main paper. Search space $z$ means latent policy adaptation leveraging the policy space $\{\pi_z\}$ constructed by the BFM. Search space $a$ denotes fine-tuning in action space (i.e., updating all policy
    network parameters).}
    \label{tab:appendix.algorithms}
\end{table}

We evaluated these variants across four DMC domains (Quadruped, Pointmass, Cheetah, Walker) using FB, HILP and PSM, and on HumEnv using FB-CPR, each experiment conducted over five random seeds. The results are shown in Figure~\ref{fig:hilp_ablation}, ~\ref{fig:fb_ablation}, ~\ref{fig:psm_ablation} and ~\ref{fig:fbcpr_ablation}.

\textbf{Zero-Shot Initialization (no-zs-init):} Removing zero-shot initialization consistently degraded early-stage performance across all methods and domains, with the only exception being ReLA-a with FB and PSM on pointmass. The benefit of zero-shot initialization is especially significant on LoLA.

\textbf{Bootstrapping (no-bootstrap):} We hypothesized that value functiom bootstrapping could help stabilizing LoLA. However, we did not notice such benefit from our ablation experiments. 

\textbf{Residual critics (no-residual):} Removing residual critics in ReLA variants and TD3-based algorithm strongly impaired the effectiveness of the algorithm. This effect was especially pronounced for ReLA-a and TD3 on DMC domains.

\begin{figure}[htbp]
    \centering
    \includegraphics[width=0.99\textwidth]{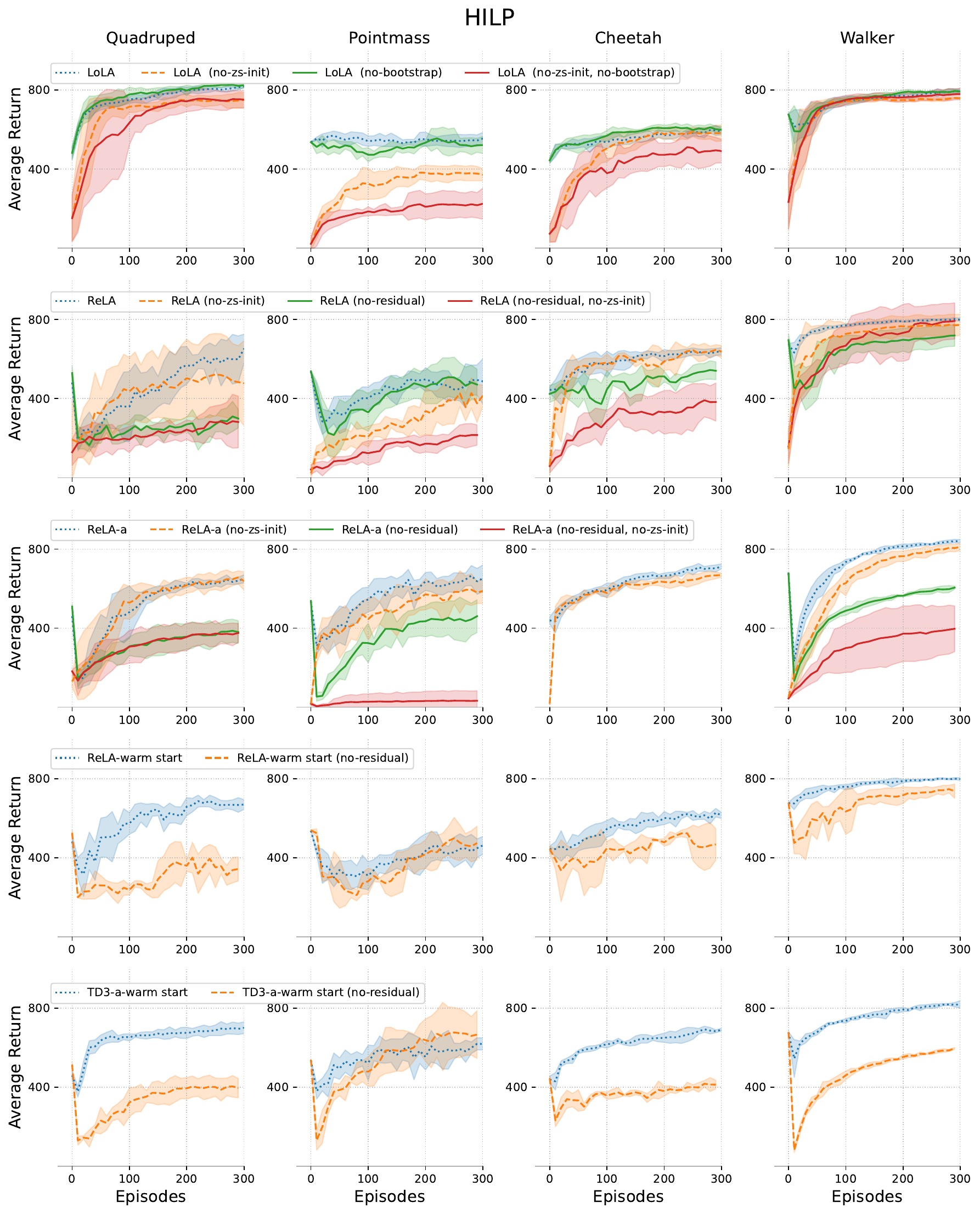} %
    \caption{Ablation studies evaluating adaptations of HILP on four DMC tasks (Quadruped, Pointmass, Cheetah, Walker). Experiments include disabling zero-shot initialization ("no-I") and/or removing residual critics ("no-R") from LoLA, ReLA, and three additional variants: (1) ReLA-a: update a instead of z in ReLA, (2) ReLA with warm start (ReLA-warm start), and (3) action-based TD3 with warm start (TD3-warm start). Shaded areas represent standard errors across 5 seeds. }
    \label{fig:hilp_ablation}
\end{figure}

\begin{figure}[htbp]
    \centering
    \includegraphics[width=0.99\textwidth]{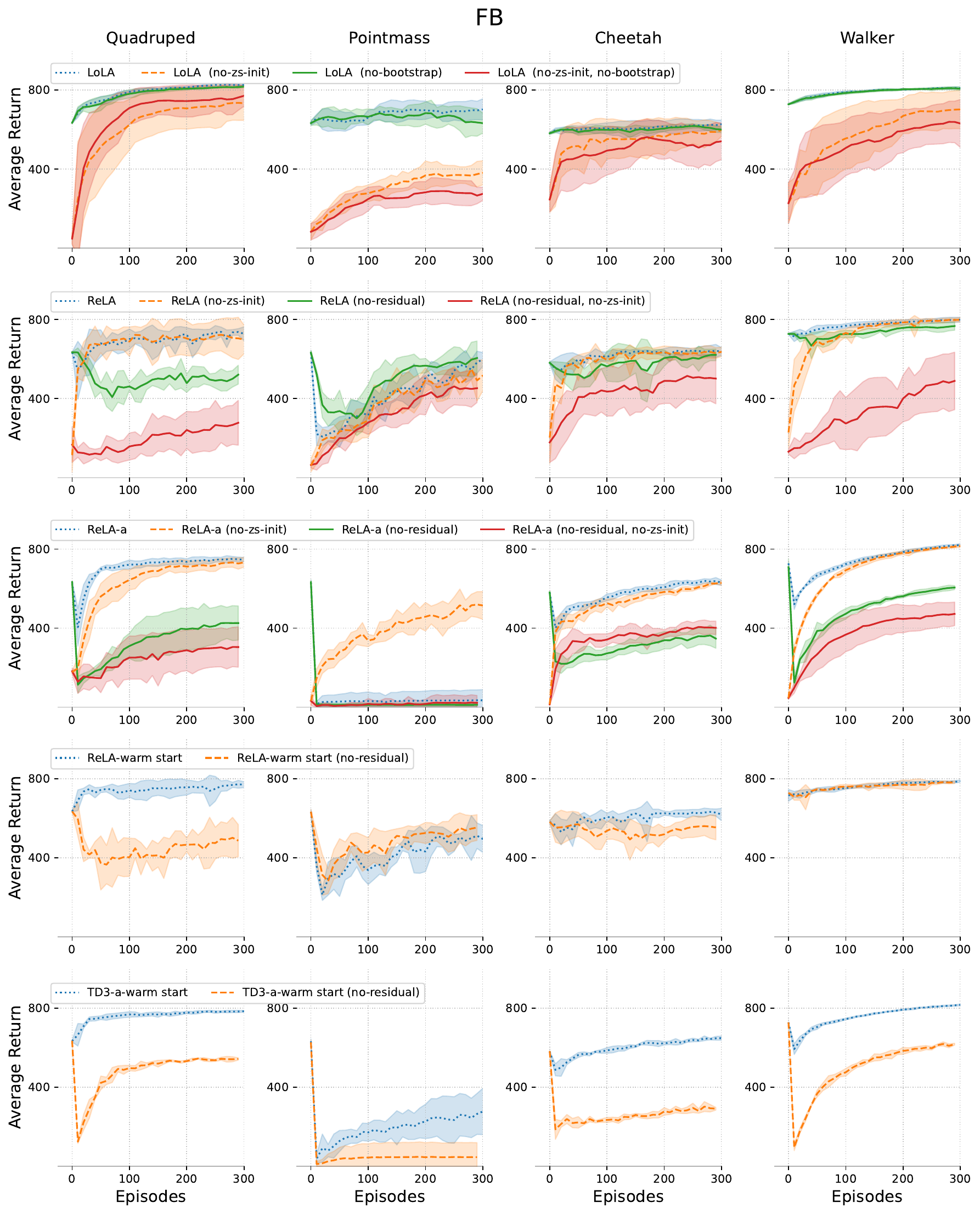} %
    \caption{Ablation studies evaluating adaptations of FB on four DMC tasks (Quadruped, Point-mass, Cheetah, Walker). Experiments include disabling zero-shot initialization ("no-I") and/or re-moving residual critics ("no-R") from LoLA, ReLA, and three additional variants: (1) ReLA-a: update a instead of z in ReLA, (2) ReLA with warm start (ReLA-warm start), and (3) action-based
TD3 with warm start (TD3-warm start). Shaded areas represent standard errors across 5 seeds.}
    \label{fig:fb_ablation}
\end{figure}

\begin{figure}[htbp]
    \centering
    \includegraphics[width=0.99\textwidth]{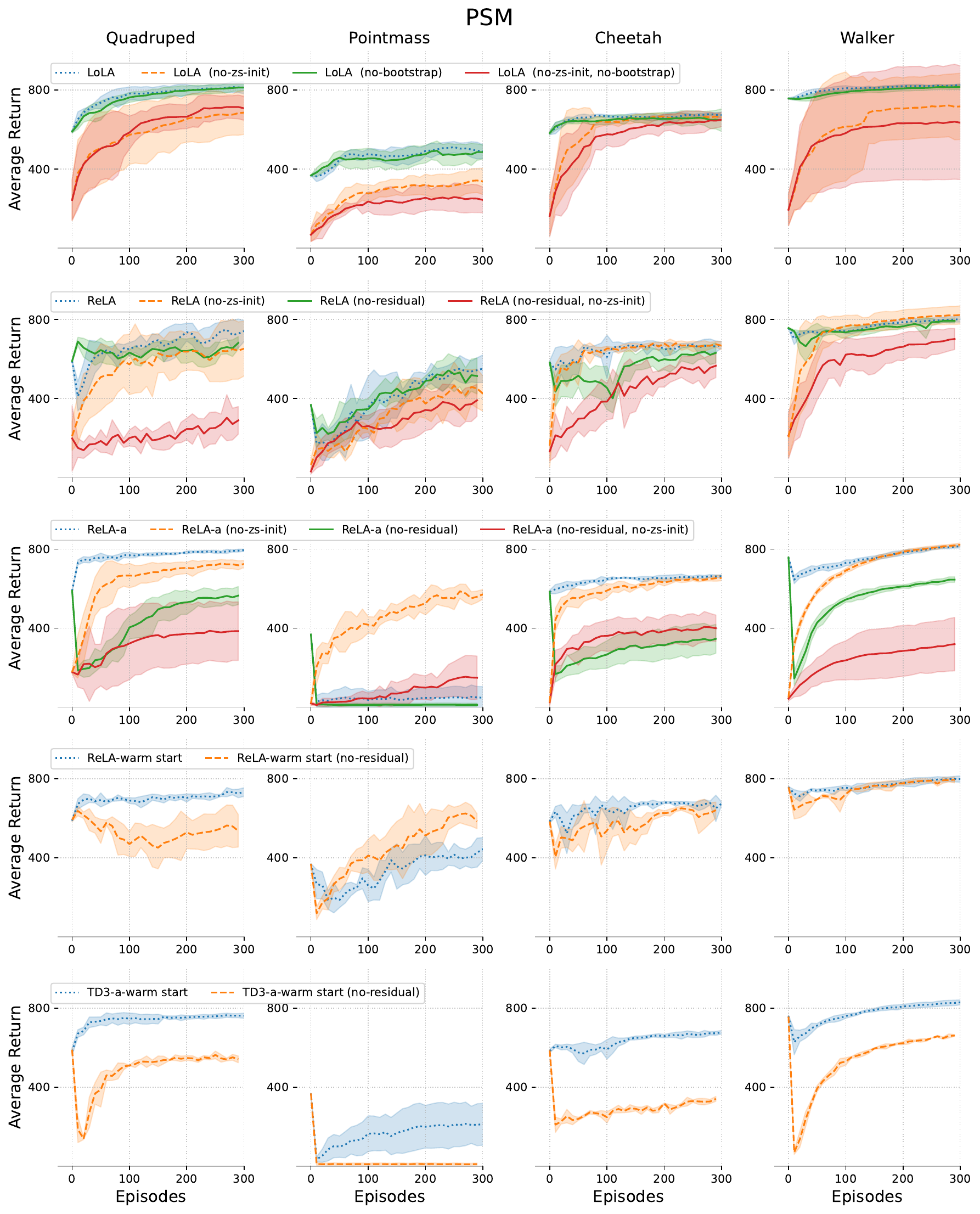} %
    \caption{Ablation studies evaluating adaptations of PSM on four DMC tasks (Quadruped, Point-
mass, Cheetah, Walker). Experiments include disabling zero-shot initialization ("no-I") and/or re-
moving residual critics ("no-R") from LoLA, ReLA, and three additional variants: (1) ReLA-a:
update action instead of z in ReLA, (2) ReLA with warm start (ReLA-warm start), and (3) action-based
TD3 with warm start (TD3-warm start). Shaded areas represent standard errors across 5 seeds.}
    \label{fig:psm_ablation}
\end{figure}

\begin{figure}[htbp]
    \centering
    \includegraphics[width=0.99\textwidth]{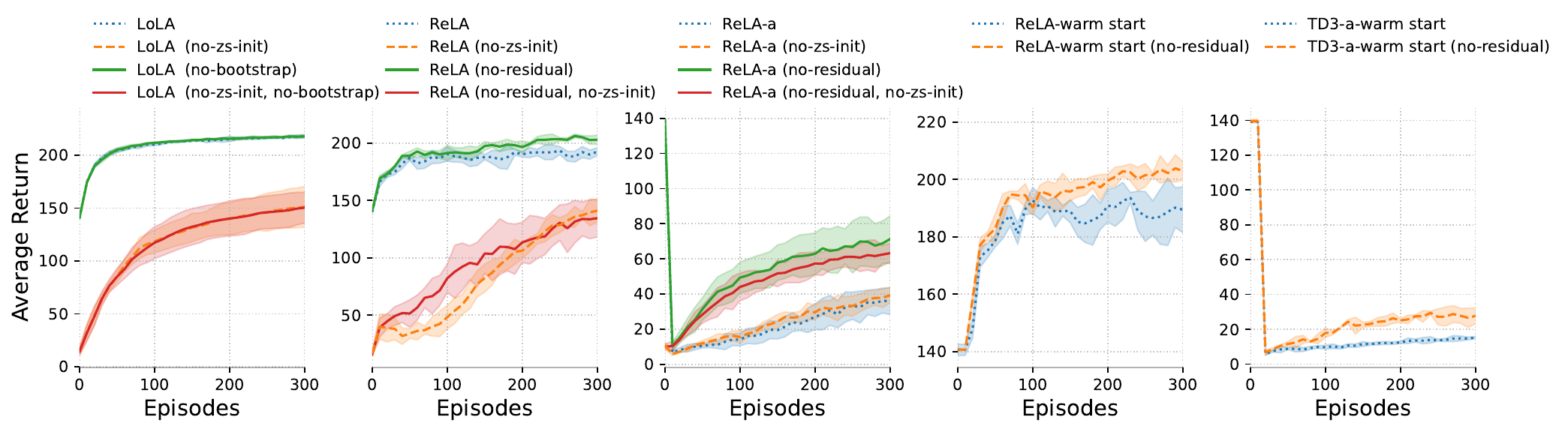} %
    \caption{Ablation studies for adaptation with FB-CPR on 45 HumEnv tasks. Experiments include disabling zero-shot initialization ("no-I") and/or re-moving residual critics ("no-R") from LoLA, ReLA, and three additional variants: (1) ReLA-a: update action instead of z in ReLA, (2) ReLA with warm start (ReLA-warm start), and (3) action-based
TD3 with warm start (TD3-warm start). Shaded areas represent standard errors across 5 seeds.}
    \label{fig:fbcpr_ablation}
\end{figure}

\end{document}